%% file: trace.tex
\documentclass{article}

\PassOptionsToPackage{numbers}{natbib}

\usepackage[main, final]{neurips_2026}

\usepackage[utf8]{inputenc} %
\usepackage[T1]{fontenc}    %
\usepackage{hyperref}       %
\usepackage{url}            %
\usepackage{booktabs}       %
\usepackage{amsfonts}       %
\usepackage{nicefrac}       %
\usepackage{microtype}      %
\usepackage[table,dvipsnames]{xcolor}         %
\usepackage{amsmath,amsthm,amssymb}
\usepackage{mathtools}
\usepackage{tikz}
\usepackage{wrapfig}
\usepackage{multirow}
\usepackage{xspace}
\usepackage[nameinlink,capitalize]{cleveref}
\usepackage{subcaption}
\usepackage{makecell}
\hypersetup{
	colorlinks,
	linkcolor={black},
} 
\usepackage{aliascnt}
\usepackage{fontawesome6}

\makeatletter
\AddToHook{cmd/appendix/before}{%
  \def\cref@section@alias{appendix}%
  \def\cref@subsection@alias{appendix}%
}
\makeatother

\author{%
  Sudip~Bhujel\textsuperscript{\textdagger}\quad
  Shanghao~Shi\textsuperscript{\textdaggerdbl}\quad
  Ruiquan~Huang\textsuperscript{\textdagger}\quad
  Ning~Zhang\textsuperscript{\textdaggerdbl}\quad
  Yang~Xiao\textsuperscript{\textdagger} \\[0.2em]
  \textsuperscript{\textdagger}Department of Computer Science, University of Kentucky\\
  \textsuperscript{\textdaggerdbl}Department of Computer Science and Engineering, Washington University in St. Louis \\[0.2em]
  \texttt{\{sudipbhujel, RuiquanHuang, xiaoy\}@uky.edu}\\
  \texttt{\{shanghao, zhang.ning\}@wustl.edu}\\[0.2em]
  \href{https://trace-rl.github.io/}{\faIcon{globe}\;https://trace-rl.github.io}\\[-1.5em]
}

\newcommand{\sysname}{{TRACE}\xspace}%

\newcolumntype{g}{>{\columncolor{gray!10}}c}
\newcolumntype{G}{>{\columncolor{gray!10}}r}

\newtheorem{theorem}{Theorem}[section]

\newaliascnt{proposition}{theorem}
\newtheorem{proposition}[proposition]{Proposition}
\aliascntresetthe{proposition}

\newaliascnt{corollary}{theorem}
\newtheorem{corollary}[corollary]{Corollary}
\aliascntresetthe{corollary}

\newaliascnt{lemma}{theorem}
\newtheorem{lemma}[lemma]{Lemma}
\aliascntresetthe{lemma}

\newaliascnt{definition}{theorem}
\newtheorem{definition}[definition]{Definition}
\aliascntresetthe{definition}

\newaliascnt{remark}{theorem}
\newtheorem{remark}[remark]{Remark}
\aliascntresetthe{remark}

\crefname{theorem}{theorem}{theorems}
\Crefname{theorem}{Theorem}{Theorems}
\crefname{proposition}{proposition}{propositions}
\Crefname{proposition}{Proposition}{Propositions}
\crefname{corollary}{corollary}{corollaries}
\Crefname{corollary}{Corollary}{Corollaries}
\crefname{lemma}{lemma}{lemmas}
\Crefname{lemma}{Lemma}{Lemmas}
\crefname{definition}{definition}{definitions}
\Crefname{definition}{Definition}{Definitions}
\crefname{remark}{remark}{remarks}
\Crefname{remark}{Remark}{Remarks}

\crefname{section}{sec.}{secs.}
\Crefname{section}{Sec.}{Secs.}
\crefname{appendix}{app.}{apps.}
\Crefname{appendix}{App.}{Apps.}
\crefname{figure}{fig.}{figs.}
\Crefname{figure}{Fig.}{Figs.}

\begin{document}

\title{
Temporal Gradient Inversion for Private Trajectory Reconstruction in Embodied Reinforcement Learning
}

\maketitle

\input{sections/abstract}

\input{sections/intro}

\input{sections/problem_formulation}

\input{sections/methodology}

\input{sections/theoretical_analysis}

\input{sections/evaluation}

\input{sections/related_work}

\input{sections/conclusion}
\input{sections/acknowledgement}

\bibliographystyle{plainnat}
\bibliography{references,software}

\input{sections/appendix}

\end{document}

%% file: sections/abstract.tex
\begin{abstract}
    Distributed learning in embodied reinforcement-learning agents offers a degree of privacy by retaining raw sensor data on-device and transmitting only policy gradients to the server. Yet temporal structure can amplify this leakage beyond single-frame attacks. We introduce \textbf{T}emporal \textbf{R}econstruction \textbf{A}ttack on \textbf{C}onsecutive \textbf{E}ncodings (TRACE), an amortized temporal gradient-inversion attack that autoregressively reconstructs the sequence of private observation-action trajectories from per-step policy-learning gradients. The attack exploits two structural signals ignored by prior single-frame methods: (i) cross-time correlation between successive embodied gradients, which we formalize via a conditional mutual-information bound, and (ii) closed-form action recovery from policy-head gradient structure, which we prove exact when standard entropy regularization is sufficiently small. On held-out embodied scenes, TRACE reaches $18.8$ dB PSNR with near-perfect action recovery at $3$--$4.5$ ms per reconstructed frame, dominating the learning-based baseline across all reconstruction metrics and exceeding optimization attacks while running orders of magnitude faster. Further evaluation demonstrates TRACE's broader applicability across recurrent, residual, and compact transformer victim architectures, multi-modal inputs, and larger discrete action spaces. Defense experiments suggest that protecting temporal gradient streams may require sequence-aware privacy mechanisms.
\end{abstract}

%% file: sections/intro.tex
\section{Introduction}
\label{sec:introduction}

Gradient inversion attacks have emerged as a fundamental privacy threat in modern machine learning. By leveraging gradients computed on private input, an adversary can reconstruct the original data with high fidelity \citep{zhu2019deep,geiping2020inverting,yin2021see}. Although initially studied in standard supervised learning settings, this privacy risk becomes particularly relevant wherever gradients are exposed. Notably, in distributed reinforcement learning (RL) scenarios such as robot fleets and shared navigation, the embodied RL agents share model/policy gradients to an external aggregating server while keeping raw environmental observations on-device.
In these settings, 
the privacy risk imposed by gradient inversion attacks becomes aggravated: 
a recovered embodied trajectory leaks the agent's time-ordered camera view and actions, such as room layout, objects, and where the agent moved, looked, and lingered, enabling activity profiling, premise mapping, and persistent monitoring by a curious server.

\begin{figure*}[t]
    \vspace{-4pt}
    \centering
    \begin{subfigure}{0.64\textwidth}
        \centering
        \includegraphics[width=\linewidth]{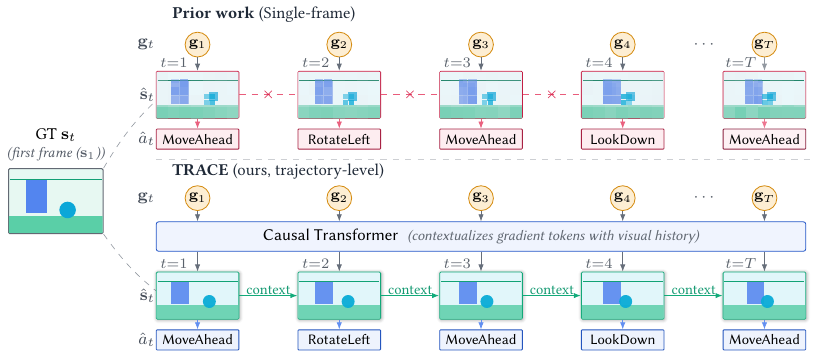}
    \end{subfigure}
    \hspace{2pt}
    \begin{subfigure}{0.34\textwidth}
        \centering
        \includegraphics[width=\linewidth]{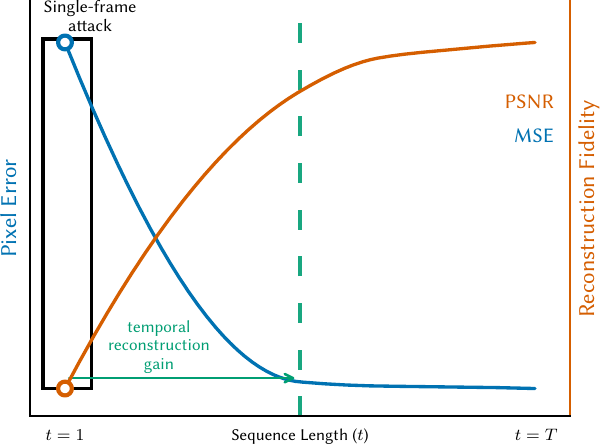}
    \end{subfigure}
    \caption{
    Illustration of trajectory reconstruction from embodied RL gradients. (Left) Prior work inverts each gradient $\mathbf{g}_t$ independently, ignoring temporal coherence; \sysname processes the full trajectory $\{\mathbf{g}_1,\ldots,\mathbf{g}_T\}$ with a causal transformer that sequentially decodes observations ($\mathbf{\hat{s}}_t$) and discrete actions ($\mathbf{\hat{a}}_t$), improving recovery accuracy. (Right) Qualitative illustration of pixel error (MSE) and reconstruction fidelity (PSNR) versus sequence length $t$: single-frame optimization attacks ($t{=}1$) incur high error, while \sysname improves with longer context---realizing temporal reconstruction gain---until returns diminish as $t \to T$ (see \Cref{rem:diminishing_returns}).
    }
    \label{fig:hero}
    \vspace{-14pt}
\end{figure*}

From the attacker's perspective, existing gradient inversion techniques target single images or static mini-batches, treating them as i.i.d. samples optimized with simple supervised losses such as cross-entropy~\citep{zhu2019deep,geiping2020inverting,yin2021see,wu2023learning}.
By contrast, embodied policy-learning traces are ordered measurements of a physical process, not i.i.d. samples. Consecutive frames captured by the agent's camera are coupled by transition dynamics: bounded translations, bounded rotations, and a largely persistent 3D scene cause nearby observations to share layout, objects, lighting, and viewpoint-dependent geometry. In actor-critic methods such as PPO~\citep{schulman2017proximal} and A2C~\citep{mnih2016asynchronous,konda1999actor}, the communicated gradients further combine policy, value, and entropy terms whose weights depend on rollout-derived advantages and returns. RL gradient leakage is therefore a trajectory-inversion problem rather than a collection of independent image inversions: the attacker must recover both private RGB observations and discrete actions from a noisy, loss-weighted gradient stream, while the same temporal coherence provides side information that can reduce per-step reconstruction ambiguity. Thus, a valid reconstruction should satisfy the environment's transition dynamics, and those dynamics also provide the structure \sysname exploits.

\textbf{This work.} We introduce \sysname (\textbf{T}emporal \textbf{R}econstruction \textbf{A}ttack on \textbf{C}onsecutive \textbf{E}ncodings), a temporal gradient inversion attack targeting embodied RL agents through near real-time \emph{inference}. Our key insight is that, in distributed learning settings for embodied agents, consecutive observations leave correlated signatures in their corresponding policy gradients. Unlike prior attacks---whether per-gradient optimization or single-frame learning-based mappers---that invert each gradient independently, \sysname amortizes the attack across the trajectory through an autoregressive model with three components: a gradient encoder maps gradient updates to compact latents, a causal transformer that propagates trajectory context, and a decoder that jointly reconstructs observations and predicts actions. When the server observes per-step gradients from consecutive viewpoints, temporal context amplifies leakage. We substantiate this claim both theoretically, through closed-form discrete-action recovery and an information-theoretic characterization of temporal gain, and empirically. Because policy memory, visual backbones, and auxiliary observations change the gradient representation, we also test whether reconstruction leakage persists beyond the CNN victim.

Our main contributions include:
\vspace{-2mm}
\begin{itemize}
    \setlength\itemsep{0.1em}
    \item We propose \sysname, an amortized trajectory-level inversion framework that uses gradient encoding, causal temporal modeling, and autoregressive decoding to recover observations/actions from gradient streams in milliseconds (\Cref{sec:overview}).
    \item We develop a theoretical foundation for actor-critic trajectory-gradient leakage, proving closed-form action recovery and characterizing temporal reconstruction gain (\Cref{sec:theoretical_analysis}).
    \item We evaluate \sysname against both optimization-based and learning-based baselines, as well as common gradient defenses, demonstrating high reconstruction fidelity and robust action recovery in realistic embodied-RL settings. The results show temporal gradient leakage transfers across unseen embodied environments, as \sysname achieves strong zero-shot reconstruction on held-out scenes and exhibits rapid improvement when adapted using only a small subset of target-domain trajectories (\Cref{sec:evaluation}). To assess the broader applicability of \sysname, we also evaluate recurrent, residual, and transformer victims, multi-modal inputs, and larger discrete action spaces (\Cref{sec:victim_variants}).
    \item We open-source our implementation of \sysname, which is available at \url{https://github.com/trace-rl/TRACE}.
\end{itemize}

%% file: sections/problem_formulation.tex
\section{Problem Formulation}
\label{sec:problem_formulation}

We adopt a standard assumption from gradient inversion research \citep{zhu2019deep,geiping2020inverting,zhao2020idlg,yin2021see,jeon2021gradient}: the model server is \emph{passive} and \emph{honest-but-curious}, faithfully aggregating while observing the client's transmitted gradient updates. 
We consider the clients are tasked with policy learning via A2C~\citep{mnih2016asynchronous} or PPO~\citep{schulman2017proximal}, two widely used actor-critic methods. At each communication round, the server broadcasts a policy model $\theta_0$ to the client. The client then interacts with its private environment using this model, computes per-step gradients from its local rollout, and sends the ordered gradients to the server. 
Given a contiguous trajectory of gradients $\mathcal{G}_{1:T} = (\mathbf{g}_1,\ldots,\mathbf{g}_T)$ from a victim client where $T$ denotes the trajectory length, the curious server seeks to reconstruct the corresponding trajectory: observations $\hat{\mathbf{s}}_{1:T}$ and actions $\hat{a}_{1:T}$, by minimizing expected reconstruction error and action prediction error.

It is worth noting that cryptographic protections, such as secure aggregation~\citep{bonawitz2017practical}, can prevent a server from seeing individual gradients. However, they rely on heavy cryptography methods and introduce extra system/security assumptions that do not apply to the ordered-gradient setting studied here. 
Under the standard white-box assumption for distributed learning, we assume the adversary (curious server) has full knowledge of model architecture, the current checkpoint $\theta_0$, the action space, and the learning rule, but not the underlying data~\citep{zhu2019deep}. 
The adversary has auxiliary trajectories from the same task family for offline training, after which inference relies solely on the observed ordered gradients.

Formally, let $\mathcal{S} = [0,1]^{3\times H_I \times W_I}$ denote the RGB observation space, with image height $H_I$ and width $W_I$, and let $\mathcal{A}=\{1,\dots,K\}$ denote the discrete action space, where $K=|\mathcal{A}|$. A summary of the key symbols used throughout is provided in \Cref{sec:notations}. For each rollout pair $(\mathbf{s}_t, a_t)$, the client computes a generalized advantage estimate $\hat{A}_t$ and return target $\hat{R}_t$, both derived from the client's local trajectory and hidden from the attacker. We adopt a unified per-step actor-critic objective
{
\small
\begin{equation}
    \label{eq:capture-loss}
    \mathcal{L}^{\mathrm{ac}}_t(\theta) = \ell_t^{\mathrm{pol}}(\theta) + c_v\big(V_\theta(\mathbf{s}_t)-\hat{R}_t\big)^2 - c_e\,\mathcal{H}\!\left(\pi_\theta(\cdot \mid \mathbf{s}_t)\right),
\end{equation}
}
where the policy term $\ell^{\mathrm{pol}}_t$ instantiates the underlying algorithm:
{
\small
\begin{equation}
    \label{eq:policy-loss-instances}
    \ell^{\mathrm{pol}}_t(\theta) \;=\;
    \begin{cases}
        -\hat{A}_t \log \pi_\theta(a_t \mid \mathbf{s}_t), & \text{A2C},\\[4pt]
        -\min\!\Big( r_t(\theta)\hat{A}_t,\; \mathrm{clip} (r_t(\theta),1-\epsilon,1+\epsilon)\hat{A}_t \Big), & \text{PPO},
    \end{cases}
    \quad \text{where} \quad
    r_t(\theta) = \frac{\pi_\theta(a_t \mid \mathbf{s}_t)} {\pi_{\theta_0}(a_t \mid \mathbf{s}_t)}.
\end{equation}
}
The attacker observes the corresponding per-step communicated gradient
\begin{equation}
    \label{eq:per-step-grad}
    \mathbf{g}_t = \nabla_\theta \mathcal{L}^{\mathrm{ac}}_t(\theta)\big|_{\theta=\theta_0} \in \mathbb{R}^d, \qquad \mathbf{s}_t \in \mathcal{S},\; a_t \in \mathcal{A}.
\end{equation}
\textbf{Modeling assumption.} We treat $\mathbf{g}_t$ as the per-step gradient at the broadcast parameters $\theta_0$, consistent with per-step policy-gradient sharing in collaborative policy learning~\citep{li2024privacy,he2025gradient}.
Multi-step client policy updates, defined as $\Delta \theta = \theta_0 - \theta_{\mathrm{client}}$, combine gradients computed along a changing optimization path and may reduce the temporal structure available to the attacker~\citep{he2025gradient}. However, we focus on the per-step setting in this work.
Here, $\pi_\theta(\cdot \mid \mathbf{s}_t)$, $V_\theta(\mathbf{s}_t)$, and $\mathcal{H}(\pi_\theta(\cdot \mid \mathbf{s}_t))$ denote the policy, value, and entropy, so the three terms in \Cref{eq:capture-loss} carry advantage-weighted policy, return-regression, and entropy-regularization signals from both actor and critic branches. A2C~\citep{mnih2016asynchronous} and PPO~\citep{schulman2017proximal} share the same critic and entropy terms and differ only in $\ell^{\mathrm{pol}}_t$. Later in \Cref{prop:local_ppo_capture}, we show their policy gradients coincide at $\theta_0$, so our attack applies to both without modification.

%% file: sections/methodology.tex
\section{\sysname Design}
\label{sec:overview}

\begin{wrapfigure}{r}{0.7\textwidth}
    \centering
    \vspace{-40pt}
    \includegraphics[width=\linewidth]{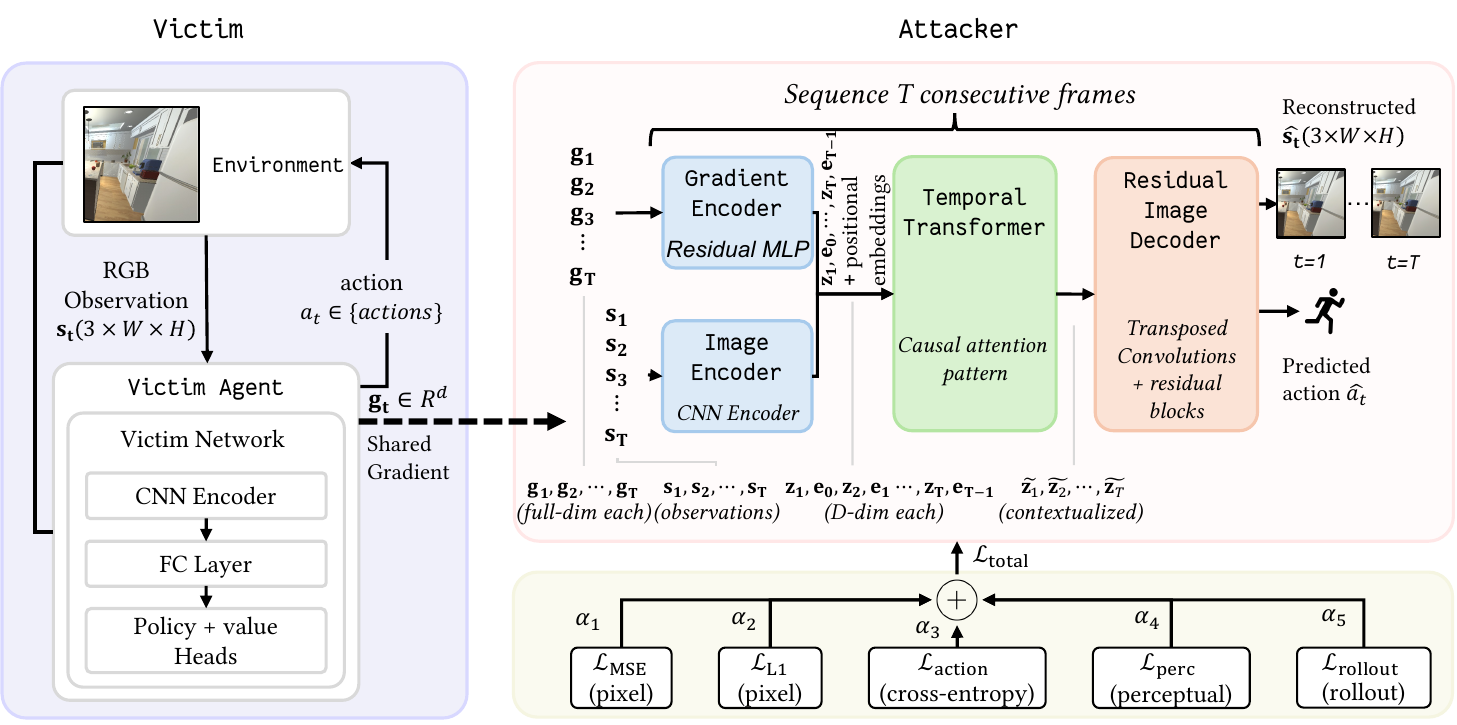}
    \vspace{-12pt}
    \caption{
    \sysname attack overview.
    }
    \label{fig:system}
    \vspace{-10pt}
\end{wrapfigure}
The core intuition behind \sysname is twofold. 
First, gradient updates from policy learning tend to reflect the agent's current observation: because the policy gradient $\nabla_\theta \log \pi_\theta(a_t \mid \mathbf{s}_t)$ passes through $\mathbf{s}_t$ via intermediate activations, the gradient vector retains information about the input observation that may be exploited for reconstruction. 
Second, embodied RL trajectories exhibit strong temporal coherence---consecutive observations share scene geometry, lighting, and object layout---providing cross-step priors that single-frame inversion methods cannot exploit. 
As shown in \Cref{fig:system}, \sysname leverages both signals jointly: 
a gradient encoder and an image encoder extract a per-step latent from the gradient updates $\mathbf{g}_{1:T}$; a temporal (causal) transformer propagates scene context across the trajectory. 
Finally, a residual decoder reconstructs $\hat{\mathbf{s}}_t$ in the pixel space from the contextualized latent and predicts the action $\hat{a}_t$.

\subsection{Attacker Model}

The attacker's model has three components: i) a gradient encoder that compresses the million-dimensional update into a low-dimensional latent representation; ii) a temporal transformer that contextualizes each latent using causal gradient and image-context tokens; and iii) a decoder that reconstructs the image and predicts the corresponding action. 

\vspace{-2pt}
\textbf{Encoder.}
The encoder component comprises a gradient encoder $\phi_\mathrm{enc} : \mathbb{R}^d \to \mathbb{R}^D$ and an image encoder $\phi_\mathrm{img}: [0, 1]^{3 \times H_I \times W_I} \to \mathbb{R}^D$, where $D \ll d$ is the shared latent dimension. 
The gradient encoder maps the high-dimensional raw gradients $\{\mathbf{g}_1, \dots, \mathbf{g}_T\}$ to compact representations $\mathbf{z}_t = \phi_{\mathrm{enc}}(\mathbf{g}_t) \in \mathbb{R}^D$.
The image encoder produces context embeddings. During the teacher-forced warmup: ground-truth observations provide $\mathbf{e}_t = \phi_{\mathrm{img}}(\mathbf{s}_t)$ for $t \geq 1$ while a learned start token $\mathbf{e}_0 \in \mathbb{R}^D$, initialized from $\mathcal{N}(\mathbf{0}, \sigma_{\mathrm{init}}^2\mathbf{I})$, is a shared marker in the parallel sequence. This warmup path forms the interleaved transformer input:
\begin{equation}
    \small
    \label{eq:interleaved_sequence}
    \mathbf{X}^{\mathrm{TF}} = [\mathbf{z}_1,\, \mathbf{e}_0,\, \mathbf{z}_2,\, \mathbf{e}_1,\, \ldots,\, \mathbf{z}_T,\, \mathbf{e}_{T-1}] \in \mathbb{R}^{2T \times D}.
\end{equation}
Because attention is causal, the first reconstruction is conditioned on $\mathbf{z}_1$ and cannot attend to $\mathbf{e}_0$. Similarly, sequential inference begins directly with the first observed gradient.

\vspace{-2pt}
\textbf{Temporal Transformer.}
The causal transformer $\phi_{\mathrm{temp}} : \mathbb{R}^{2T \times D} \to \mathbb{R}^{2T \times D}$ uses $N$-layer multi-head self-attention with a lower-triangular causal mask $\mathbf{M} \in \{0,1\}^{2T \times 2T}$, $M_{ij}=\mathbf{1}[j\leq i]$. In the teacher-forced warmup sequence of \Cref{eq:interleaved_sequence}, the representation used to decode timestep $t$ is taken from the gradient-token position $\mathbf{z}_t$. Because $\mathbf{e}_{t-1}$ appears immediately after $\mathbf{z}_t$, this parallel warmup path uses lagged image context: $\mathbf{e}_{t-1}$ is available only to later gradient tokens. The sequential scheduled-sampling, rollout, and inference protocol is described in \Cref{sec:autoregressive-training}: after decoding step $t$, the encoded ground-truth or predicted frame is appended as context for later steps. Thus, inference is amortized per step but sequential across $T$ decoding steps.
The model uses gated residual connections with learnable scalars $\alpha_{\mathrm{attn}}^{(n)}$ and $\alpha_{\mathrm{ffn}}^{(n)}$, initialized so residual branches are initially down-weighted:
\vspace{-2pt}
\begin{equation}
    \small
    \label{eq:gated_residual}
    \mathbf{x}^{(n+1)} = \mathbf{x}^{(n)} + \sigma(\alpha_{\mathrm{attn}}^{(n)}) \cdot \mathrm{Attn}^{(n)}(\mathbf{x}^{(n)}) + \sigma(\alpha_{\mathrm{ffn}}^{(n)}) \cdot \mathrm{FFN}^{(n)}(\mathbf{x}^{(n)}).
\end{equation}
The contextualized gradient representations $\tilde{\mathbf{z}}_t$ are extracted from the output of the transformer.

\vspace{-2pt}
\textbf{Decoder.}
The decoder $\phi_{\mathrm{dec}} : \mathbb{R}^D \to [0,1]^{3 \times H_I \times W_I} \times \Delta^{K-1}$ comprises two heads sharing the contextualized latent $\tilde{\mathbf{z}}_t$: an image head $\phi_{\mathrm{img,dec}}$ that upsamples $\tilde{\mathbf{z}}_t$ to a reconstructed observation $\hat{\mathbf{s}}_t$ via transposed convolutions with a sigmoidal output, and an action head $\phi_{\mathrm{act}}$---a two-layer MLP---that produces a categorical distribution over actions:
\begin{equation}
    \small
    \label{eq:decoder}
    \hat{\mathbf{s}}_t = \phi_{\mathrm{img,dec}}(\tilde{\mathbf{z}}_t), \qquad
    \hat{p}_t = \mathrm{softmax}\!\bigl(\phi_{\mathrm{act}}(\tilde{\mathbf{z}}_t)\bigr), \qquad
    \hat{a}_t = \operatorname*{arg\,max}_{k} [\hat{p}_t]_k.
\end{equation}

\input{sections/exposure_bias}

\textbf{Training Objective.}
\label{sec:training_objective}
Lastly, we train our attacker model using a weighted multi-term objective, a standard approach in gradient inversion~\citep{geiping2020inverting,yin2021see,jeon2021gradient,hatamizadeh2022gradvit}, combining complementary loss functions:
\begin{equation}
    \small
    \label{eq:total-loss}
    \mathcal{L}_{\mathrm{total}} = 
    \alpha_1 \mathcal{L}_{\mathrm{MSE}} + \alpha_2  \mathcal{L}_{\mathrm{L1}} + \alpha_3 \mathcal{L}_{\mathrm{act}} + \alpha_4 \mathcal{L}_{\mathrm{perc}}
    + \alpha_5 \mathcal{L}_{\mathrm{rollout}}
    .
\end{equation}
where $\mathcal{L}_{\mathrm{MSE}}$ and $\mathcal{L}_{\mathrm{L1}}$ evaluate pixel fidelity loss, $\mathcal{L}_{\mathrm{act}}$ action supervision loss, $\mathcal{L}_{\mathrm{perc}}$ perceptual quality loss, and $\mathcal{L}_{\mathrm{rollout}}$ rollout-supervision loss that exposes the model to its own predictions during training to mitigate exposure bias. $\alpha$'s are the weights. We provide a full derivation in \Cref{sec:derivation_of_training_objective}.

%% file: sections/exposure_bias.tex
\subsection{Autoregressive Training}
\label{sec:autoregressive-training}

\begin{wrapfigure}{r}{0.65\linewidth}
    \vspace{-30pt}
    \centering
    \begin{subfigure}{0.48\linewidth}
        \centering
        \includegraphics[width=\linewidth]{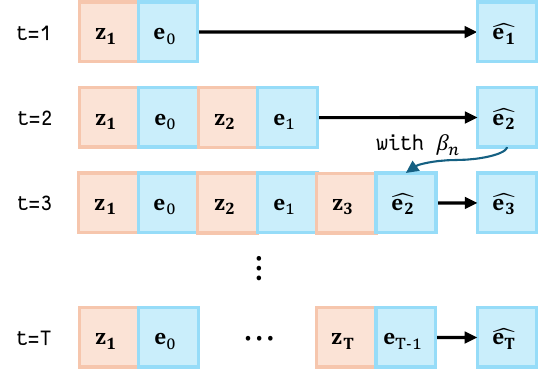}
        \caption{
        }
        \label{fig:training}
    \end{subfigure}
    \hspace{4pt}
    \begin{subfigure}{0.48\linewidth}
        \centering
        \includegraphics[width=\linewidth]{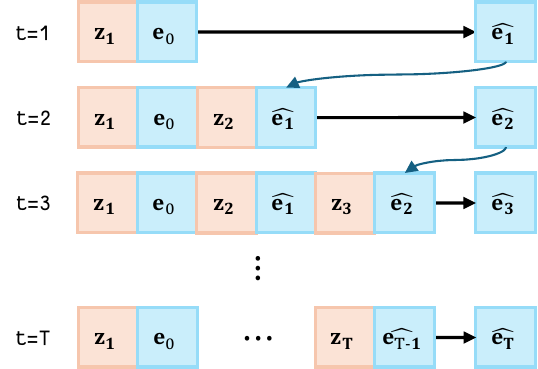} 
        \caption{
        }
        \label{fig:inference}
    \end{subfigure}
    \vspace{-15pt}
    \caption{Train--inference protocol of \sysname. (a) During training, after warmup, the next image-context token is encoded from either the ground-truth observation or the model's prediction with probability $\beta_n$. (b) During inference, each decoded observation is encoded and appended as image context for subsequent timesteps.}
    \label{fig:train_inf}
    \vspace{-12pt}
\end{wrapfigure}
\Cref{fig:train_inf} illustrates the training and inference processes of \sysname. Unlike video-generation settings such as MAGI~\citep{zhou2025taming},  which generate future frames from prior visual context alone, \sysname conditions on per-step gradients to invert private trajectories.
Because \sysname models trajectories sequentially, generating the state at time $t$ depends on the previous predictions. During the initial teacher-forced warmup, the model uses the ground-truth image context in the lagged parallel path of \Cref{eq:interleaved_sequence}. This parallelizes the forward pass and stabilizes early training by preventing compounding errors.

During training, however, teacher forcing introduces \emph{exposure bias}~\citep{ranzato2015sequence}: the model relies on an ideal historical context unavailable during inference. To address this discrepancy, we employ two sequential strategies. First, we employ \emph{scheduled sampling}~\citep{bengio2015scheduled} to gradually replace ground-truth frames with the model's own predictions when constructing future image--context tokens. 
The epoch-wise replacement probability is
{
\small
\begin{equation}
    \label{eq:scheduled_sampling}
    \beta_n =
    \begin{cases}
        0, & n \le \tau,\\
        \beta_{\mathrm{start}}
        + (\beta_{\mathrm{end}}-\beta_{\mathrm{start}})
        \dfrac{n-\tau}{N-\tau}, & n > \tau,
    \end{cases}
\end{equation}
}%
where $n$ denotes the current epoch, $\tau$ is the warmup duration, $N$ is the total number of epochs, and $(\beta_{\mathrm{start}}, \beta_{\mathrm{end}})$ are schedule endpoints. Hence, training is fully teacher-forced for $n \le \tau$, then linearly transitions to stronger self-conditioned context usage.
Second, we add an auxiliary short-horizon \emph{rollout loss}:
from a random start index, we unroll $R$ autoregressive steps, supervise the predictions, where $R$ is the rollout horizon, and feed each predicted frame back as context. Feedback frames are detached to preserve supervision under self-generated context while avoiding full gradient propagation through the feedback path. Inference follows the same sequential feedback pattern over all $T$ timesteps, yielding per-step decoding instead of a parallel full-trajectory pass.

%% file: sections/theoretical_analysis.tex
\section{Theoretical Analysis}
\label{sec:theoretical_analysis}

We focus on two questions: (i) What information is revealed by the per-step policy gradient? (ii) Why does temporal conditioning enhance trajectory reconstruction? Full proofs are deferred to \Cref{sec:theoretical_analysis_derivation}.

\subsection{Gradient Leakage and Action Recoverability}
\label{sec:gradient_leakage}

We first establish that at the broadcast checkpoint $\theta_0$, A2C~\citep{mnih2016asynchronous} and PPO~\citep{schulman2017proximal} produce the same communicated gradient, thereby yielding a clean three-term decomposition.

\begin{proposition}[A2C--PPO gradient equivalence and decomposition]
    \label{prop:local_ppo_capture}
    For each timestep $t$, there exists a neighborhood $\mathcal{U}_t$ of $\theta_0$ such that
    $\mathrm{clip}(r_t(\theta),1-\epsilon,1+\epsilon) = r_t(\theta)$ for all $\theta \in \mathcal{U}_t$.
    Consequently, $\nabla_\theta \ell^{\mathrm{pol,PPO}}_t(\theta)\big|_{\theta_0} = \nabla_\theta \ell^{\mathrm{pol,A2C}}_t(\theta)\big|_{\theta_0}$, so both algorithms produce the \emph{same} communicated gradient at $\theta_0$, which decomposes as
    {
        \small
        \begin{equation}
            \label{eq:gradient_decomposition}
            \mathbf{g}_t
            =
            \underbrace{-\hat{A}_t \nabla_\theta \log \pi_\theta(a_t \mid \mathbf{s}_t)\big|_{\theta_0}}_{\text{policy gradient}\;\mathbf{g}_t^{(\pi)}}
            \;+\;
            \underbrace{2c_v\bigl(V_{\theta_0}(\mathbf{s}_t)-\hat{R}_t\bigr)\,\nabla_\theta V_\theta(\mathbf{s}_t)\big|_{\theta_0}}_{\text{value gradient}\;\mathbf{g}_t^{(v)}}
            \;\underbrace{-\,c_e\,\nabla_\theta \mathcal{H}\!\bigl(\pi_\theta(\cdot \mid \mathbf{s}_t)\bigr)\big|_{\theta_0}}_{\text{entropy gradient}\;\mathbf{g}_t^{(e)}}.
        \end{equation}
    }
\end{proposition}
This decomposition isolates three leakage channels. The policy term carries action information (\Cref{thm:action_identifiability}), while value and entropy terms encode additional state-dependent structure.

Next, we show that the action label is recoverable from the policy-head gradient component $\mathbf{g}_t^{(\pi)}$.
\begin{theorem}[Discrete-Action recovery from the policy-head gradient structure]
    \label{thm:action_identifiability}
    Let $\boldsymbol{\ell}_t = \mathbf{W}_a^\top \mathbf{h}(\mathbf{s}_t)\in\mathbb{R}^K$ be the policy logits, $\pi_\theta(\cdot\mid\mathbf{s}_t)=\mathrm{softmax}(\boldsymbol{\ell}_t)$ and $\pi_k := \pi_\theta(k\mid\mathbf{s}_t)$.
    Assume: \textbf{(A1)} the actor and critic heads share no parameters, so $V_\theta$ does not depend on $\mathbf{W}_a$; \textbf{(A2)} $\hat{A}_t \neq 0$; \textbf{(A3)} $\mathbf{h}(\mathbf{s}_t)\in\mathbb{R}_{\ge 0}^{d_h}$ with $\lVert\mathbf{h}(\mathbf{s}_t)\rVert_1>0$ (e.g., a ReLU activation); \textbf{(A4)} the policy is non-degenerate at $\mathbf{s}_t$, $\pi_{a_t} < 1$. Let $\widetilde{\sigma}_k \;:=\; \sum_{i=1}^{d_h}\bigl[\nabla_{\mathbf{W}_a}\mathcal{L}^{\mathrm{ac}}_t(\theta)\bigr]_{i,k}\Big|_{\theta_0}$ denote the column sums of the observed policy-head gradient.

    \textbf{(i) Exact recovery without entropy regularization.}
    If $c_e=0$, then $\widetilde{\sigma}_k = \sigma_k := \hat{A}_t\,\lVert\mathbf{h}(\mathbf{s}_t)\rVert_1\,\bigl(\pi_k - \mathbf{1}[k=a_t]\bigr)$, and the true action is recovered exactly by
    {
    \small
    \begin{equation}
        \label{eq:action_recovery}
        a_t = \arg\min_k \widetilde{\sigma}_k \text{ if } \hat{A}_t > 0, \qquad a_t = \arg\max_k \widetilde{\sigma}_k \text{ if } \hat{A}_t < 0.
    \end{equation}
    }%
    For $K \ge 3$, the sign-agnostic rule $a_t = \operatorname*{arg\,max}_{k}|\widetilde{\sigma}_k|$ is equivalent to \eqref{eq:action_recovery}; for $K=2$ the two magnitudes coincide and \eqref{eq:action_recovery} must be used, with $\mathrm{sign}(\hat{A}_t)$ inferred from cross-head information such as the value-head gradient (see \Cref{sec:proof_action_recovery}).

    \textbf{(ii) Robust recovery under entropy regularization.}
    Let $L_t := \max_k \ell_{t,k} - \min_k \ell_{t,k}$ denote the logit spread, $\pi_{\max} := \max_k \pi_k$, and define the \emph{identifiability gap}
    $\Delta_t \;:=\; |\hat{A}_t|\,\lVert\mathbf{h}(\mathbf{s}_t)\rVert_1\,\bigl(1 - \pi_{a_t} + \min_{k\neq a_t}\pi_k\bigr) \;>\;0,$
    which is strictly positive under (\textbf{A4}). If $c_e>0$ and
    {
    \vspace{-6pt}
    \small
    \begin{equation}
        \label{eq:entropy_safety}
        2\,c_e\, L_t\, \pi_{\max} \;<\; \frac{\Delta_t}{\lVert\mathbf{h}(\mathbf{s}_t)\rVert_1},
    \end{equation}
    }
    then \eqref{eq:action_recovery} still recovers $a_t$ exactly.
\end{theorem}

\begin{remark}[Necessity of (\textbf{A4}) and deterministic-policy degeneracy]
    \label{rem:deterministic_policy}
    Assumption (\textbf{A4}) is essential: in the one-hot limit $\pi_{a_t}\!\to\!1$, $\pi_k\!\to\!0$ for $k\neq a_t$, \emph{every} column sum vanishes ($\sigma_k\equiv 0$), $\Delta_t = 0$, and \eqref{eq:action_recovery} is ill-posed because the $-\hat{A}_t\log\pi_\theta(a_t\mid\mathbf{s}_t)$ contribution to $\nabla_{\mathbf{W}_a}$ is zero. Thus, the action is unrecoverable from the policy-head gradient block alone for a fully deterministic policy, whereas moderately confident PPO/A2C policies keep $\Delta_t$ bounded away from zero.
\end{remark}

\vspace{-4pt}
\Cref{thm:action_identifiability} gives a closed-form, learning-free recovery rule and shows that it remains exact when entropy regularization is small relative to the identifiability gap. Our default experiments use $K{=}5$ actions, with additional evaluation at $K{=}10$, both in the $K\ge3$ regime. The near-perfect action recovery in our experiments indicates that the evaluated checkpoints avoid the degenerate regime of \Cref{rem:deterministic_policy}; by \Cref{prop:local_ppo_capture}, the same advantage-weighted policy-gradient structure remains identical for both A2C and PPO at $\theta_0$.

\subsection{Temporal Conditioning Benefit}
\label{sec:temporal_benefit}

We formalize the non-negativity of temporal gain using a mutual-information analysis.

\begin{definition}[Temporal-context gain]
    \label{def:info_gain}
    The temporal-context gain at step $t$ is the conditional mutual information:
    {
    \small
    \begin{equation}
        \label{eq:info_gain}
        \Delta I_t \;:=\; I(\mathbf{s}_t;\, \mathbf{g}_{1:t-1} \mid \mathbf{g}_t) \;=\; H(\mathbf{s}_t \mid \mathbf{g}_t) - H(\mathbf{s}_t \mid \mathbf{g}_{1:t}) \ge 0,
    \end{equation}
    }
    i.e., the residual entropy of $\mathbf{s}_t$, after the current gradient $\mathbf{g}_t$ has been observed, that is resolved by additionally conditioning on the past gradient stream $\mathbf{g}_{1:t-1}$.
\end{definition}

We first consider a tractable special case: under a Markov victim policy and one-step (TD($0$)) advantage--return estimators with stop-gradient on the bootstrap value, $\mathbf{g}_t$ depends on the trajectory only through $(\mathbf{s}_t, a_t)$ and one-step noise, giving $\mathbf{g}_t \perp\!\!\!\perp \mathbf{g}_{1:t-1} \mid \mathbf{s}_t$ (formalized as \Cref{lem:gradient_sufficiency}). Generalized Advantage Estimation (GAE) with $\lambda{>}0$ and rollout-level normalization fall outside this case. The resulting scaling, which depends on the specific algorithm, is reported empirically in \Cref{sec:evaluation}.

\begin{theorem}[Temporal reconstruction gain]
    \label{thm:temporal_gain}
    Let $(\mathbf{s}_1, \ldots, \mathbf{s}_T)$ be a Markov Decision Process (MDP) trajectory with transition kernel {\small$\mathbf{s}_{t+1} \sim P(\cdot \mid \mathbf{s}_t, a_t)$}, and let {\small$\mathrm{MSE}^*(Y) := \min_{f}\mathbb{E}\!\left[\lVert\mathbf{s}_t - f(Y)\rVert_2^2\right]$} denote the Bayes-optimal mean squared reconstruction error of $\mathbf{s}_t$ from observation $Y$. 
    
    \textbf{(i) Non-negativity:} {\small$\Delta I_t \geq 0$} for all $t$, by non-negativity of conditional mutual information. By the Law of Total Variance, {\small$\mathrm{MSE}^*(\mathbf{g}_{1:t}) \leq \mathrm{MSE}^*(\mathbf{g}_t)$}, i.e., temporal conditioning cannot hurt. 
    
    \textbf{(ii) Exact identity under conditional sufficiency:} under \Cref{lem:gradient_sufficiency} (Markov victim policy and causally adapted, un-normalized one-step advantage/return estimators), and condition throughout on $\theta_0$. By \Cref{lem:gradient_sufficiency} we have the Markov chain $\mathbf{g}_{1:t-1} \;\to\; \mathbf{s}_t \;\to\; \mathbf{g}_t,$ and therefore
    {
        \small
        \begin{equation}
            \label{eq:temporal_identity}
            \Delta I_t
            \;=\;
            I(\mathbf{s}_t;\, \mathbf{g}_{1:t-1})
            \;-\;
            I(\mathbf{g}_t;\, \mathbf{g}_{1:t-1}) \;\leq\; I(\mathbf{s}_t;\, \mathbf{g}_{1:t-1})
        \end{equation}
    }%
    Thus, temporal conditioning helps exactly to the extent that the past gradient stream contains information about $\mathbf{s}_t$ that is not already captured by the current gradient $\mathbf{g}_t$.
\end{theorem}

\vspace{-2pt}
Part~(ii) is exact only under the conditions of \Cref{lem:gradient_sufficiency}. Both GAE with $\lambda > 0$ and rollout-level normalization violate these conditions by introducing cross-temporal coupling; the scaling behavior observed under the practical PPO and A2C estimators is therefore discussed informally in \Cref{rem:diminishing_returns} and validated empirically in \Cref{sec:evaluation}.

%% file: sections/evaluation.tex
\section{Experimental Evaluation}
\label{sec:evaluation}

\textbf{Experimental Setup.}
We evaluate \sysname against victim policies trained on a point-goal navigation task in the AI2-THOR simulator~\citep{kolve2017ai2}. Unless otherwise specified, the victim observes RGB frames downsampled to $84 \times 84$ and acts in a five-action discrete space: \texttt{MoveAhead} (with a grid step size of 0.25 meters), \texttt{RotateLeft} ($15^\circ$), \texttt{RotateRight} ($15^\circ$), \texttt{LookDown} ($15^\circ$), and \texttt{LookUp} ($15^\circ$).
The default victim policy is implemented as a shared-CNN actor-critic~\citep{konda1999actor} with separate policy and value heads. It is trained with either PPO~\citep{schulman2017proximal} or A2C~\citep{mnih2016asynchronous}, employing GAE advantages as formalized in \Cref{sec:problem_formulation}. Both methods employ the same critic, entropy regularization, and gradient-norm bound. Unless otherwise indicated, we report mean $\pm$ one standard deviation per image.
We defer the full definition of evaluation metrics to \Cref{sec:evaluation_metrics} with additional results to \Cref{sec:add_experimental_details}.

\subsection{Comparison with Baselines}
\label{sec:baselines}

\Cref{fig:inversion_ppo} shows an example of trajectory inversion result generated by \sysname\ using $T{=}8$ frames against PPO-based victim, where $T$ is the trajectory length.

\begin{figure}[t]
    \centering
    \includegraphics[width=.95\linewidth]{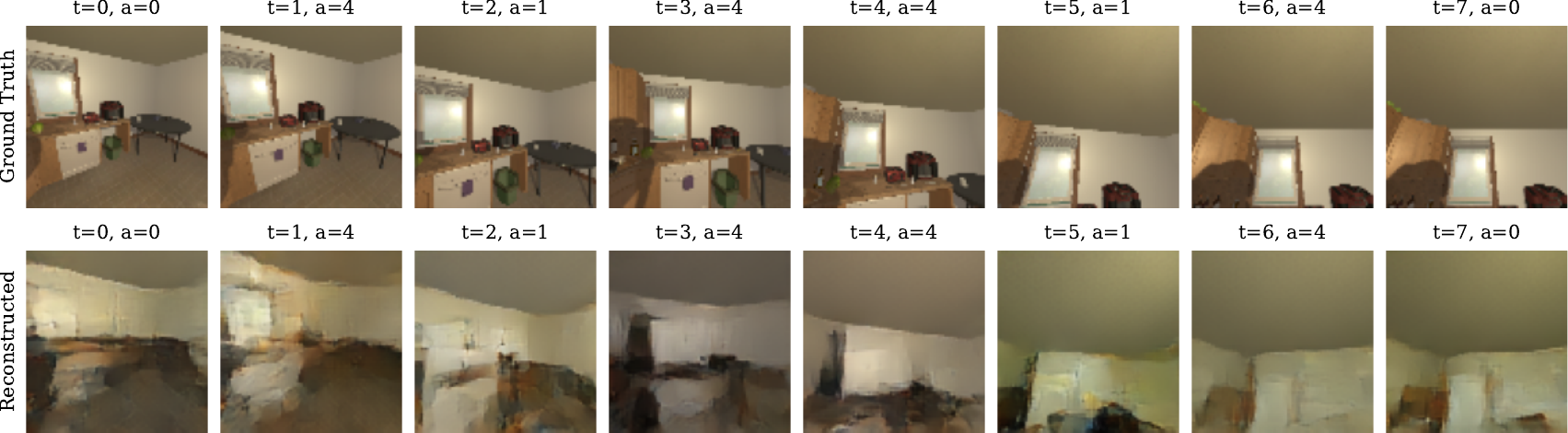}
    \vspace{-2pt}
    \caption{Ground-truth and \sysname reconstructions over $T{=}8$ PPO trajectory steps.}
    \label{fig:inversion_ppo}
    \vspace{-12pt}
\end{figure}

We further compare \sysname against DLG~\citep{zhu2019deep}\footnote{\url{https://github.com/mit-han-lab/dlg}}, Inverting Gradients (IG)~\citep{geiping2020inverting}\footnote{\url{https://github.com/JonasGeiping/invertinggradients}}, and Learning to Invert (LtI)~\citep{wu2023learning}\footnote{\url{https://github.com/wrh14/Learning_to_Invert}}. Because these baselines are designed for supervised classification, we adapt them to our RL setting by replacing the classification loss with a simplified actor-critic surrogate: policy cross-entropy combined with value regression, using fixed coefficients.
DLG jointly optimizes a dummy image and soft action distribution via L-BFGS for $300$ iterations per step. IG first recovers the discrete action analytically from the policy-head gradient (analogous to iDLG~\citep{zhao2020idlg}), then optimizes a Gaussian-initialized dummy image with Adam for $4{,}800$ iterations, minimizing gradient cosine similarity with anisotropic TV regularization and clamping to valid pixel bounds. LtI trains an MLP on state-gradient pairs to directly map randomly sub-sampled gradients to reconstructed images and discrete actions.
All methods are evaluated on 100 held-out trajectories (with $T{=}8$) for both PPO and A2C victims using consistent metrics.

\input{tables/baselines}

Table~\ref{tab:baselines} shows that \sysname attains the best performance across all reconstruction metrics under both PPO and A2C victims. Optimization-based baselines perform poorly: under PPO, DLG, and IG, PSNR reaches only $5.47$ and $8.73$\,dB, with action accuracies far below practical usefulness, and the trend persists under A2C ($6.11$ and $9.51$\,dB). LtI improves substantially, attaining $16.79$\,dB (PPO) and $17.65$\,dB (A2C) PSNR with near-perfect action accuracy, consistent with \Cref{thm:action_identifiability}. Adding the causal temporal transformer, \sysname further reduces MSE from $0.023$ to $0.014$ (PPO) and $0.018$ to $0.014$ (A2C) over LtI, gains $\sim 2$ and $\sim 1$\,dB PSNR respectively, and nearly halves LPIPS in both cases ($0.671$ to $0.362$ for PPO, $0.639$ to $0.358$ for A2C), confirming that exploiting inter-step dependencies yields superior fidelity. We note that \sysname at $T{=}1$ ($12.85$\,dB, \Cref{tab:ablations}) underperforms LtI's single-frame result, indicating its advantage comes from temporal conditioning rather than a stronger per-frame extractor, consistent with \Cref{thm:temporal_gain}. Both learned approaches run in milliseconds per frame, compared with tens of seconds for optimization-based methods, enabling near real-time inversion once the attack model is trained.

We report additional evaluations across victim architectures, multi-modal inputs, and larger action spaces in \Cref{sec:victim_variants}. We also examine four- and eight-step gradient averaging in \Cref{sec:aggregation} and auxiliary-data scaling in \Cref{sec:auxiliary_scaling}.

\subsection{Few-Shot Adaptation}
\label{sec:few_shot_adaptation}

\input{tables/ppo_adaptation}
To evaluate cross-scene robustness, we fine-tune the attacker on $p \in \{0, 10, 20, 30, 40, 50\}\%$ of target-domain trajectories, measuring how quickly leakage transfers and how much held-out data is needed for high-fidelity reconstruction.
As shown in \Cref{tab:ppo_fewshot}, the zero-shot attack provides a strong baseline ($p=0$) on held-out scenes, with perfect action recovery and reasonable structural fidelity (SSIM $= 0.627$, PSNR $= 18.8$ dB). Even minimal adaptation rapidly improves perception. Fine-tuning with $10\%$ of target trajectories increases PSNR by $2.1$ dB and lowers LPIPS to $0.273$. With a $50\%$ budget, reconstruction MSE drops to $0.005$, and SSIM exceeds $0.74$. These results show that \sysname transfers to unseen scenes without target-domain fine-tuning, while a small domain-specific set substantially reduces visual discrepancy, enabling high-fidelity reconstructions.
\subsection{Defense Evaluation}
\label{sec:defense_evaluation}

We evaluate the robustness of \sysname against four gradient defense mechanisms---quantization~\citep{alistarh2017qsgd}, magnitude-based pruning, additive Gaussian noise, and DP-SGD~\citep{abadi2016deep}---under both PPO and A2C victim models, with results shown in \Cref{tab:defenses}. 

\input{tables/defense}

The \emph{No Defense} row provides the clean reference under this defense-evaluation protocol. The attacker is evaluated in a \emph{zero-shot} model, i.e., trained on clean gradients and applied directly to defended ones. The trends are consistent across both algorithms. Pruning is largely ineffective: retaining only $10\%$ of gradient entries leaves PPO reconstruction nearly intact (PSNR $= 18.8$, Act.\ Acc.\ $= 99.9\%$), with A2C showing only a modest action accuracy reduction to $89.0\%$. Quantization to $8$-bit is similarly benign, but $4$-bit quantization degrades fidelity substantially (PPO PSNR drops to $12.4$, A2C to $13.3$), and at $2$-bit the attack approaches random-chance action recovery for PPO ($19.4\%$). Additive noise impedes reconstruction only at high variance ($\sigma{=}0.1$), where PSNR falls to $14.2$ (PPO) and $16.5$ (A2C) with action accuracy declining to $60$--$64\%$. DP-SGD offers the strongest protection: across all evaluated privacy budgets ($\varepsilon \in \{1, 5, 10\}$), SSIM drops to approximately $0.34$ (PPO) and $0.32$ (A2C), and action accuracy degrades to near-random levels (${\sim}20\%$). We do not assess downstream navigation utility under these defenses, as privacy-utility trade-offs require a separate analysis.

\subsection{Ablations}
\label{sec:ablations}

We conduct three ablation studies: (i) varying the trajectory length, (ii) restricting the attack to gradients originating from specific architectural components of the victim, and (iii) removing individual terms from the training objective. To isolate the effect of temporal modeling, we employ a reduced-capacity encoder by halving the hidden dimension from $d_{\mathrm{enc}}{=}2048$ to $d_{\mathrm{enc}}{=}1024$. This modification approximately halves the number of encoder parameters while leaving the transformer and decoder components unchanged. For the \textit{Comp.} (gradients from a single victim component) and \textit{Loss (w/o)} (one loss term removed) experiments, the trajectory length is fixed at $T = 8$. The corresponding results are presented in \Cref{tab:ablations}.

\input{tables/ppo_ablation}\vspace{-2pt}
\textbf{Trajectory Length ($T$).} Extending the horizon from $T{=}1$ to $T{=}8$ yields a $\sim 6$\,dB PSNR gain (from $12.85$ to $18.74$) and roughly halves LPIPS (from $0.781$ to $0.375$). Further extension continues to help modestly: $T{=}64$ reaches $20.29$\,dB PSNR and SSIM $0.666$. This trend in sequence length further supports the temporal conditioning mechanism formalized in \Cref{thm:temporal_gain}. We adopt $T{=}8$ as a balance between temporal context and optimization complexity.\footnote{Varying $T$ results in different total frame counts and, for $T{=}64$ the method transitions to RoPE~\citep{su2024roformer} (see ~\Cref{sec:attacker_impl} for implementation details). Nevertheless, the observed trend remains consistent for both PPO and A2C (see \Cref{tab:ablations_a2c}).}

\vspace{-2pt}
\textbf{Gradient Components.} Restricting the attack to the terminal policy/value \emph{heads} sharply degrades visual fidelity (PSNR $= 14.12$, LPIPS $= 0.511$), confirming that high-frequency structural cues reside in the earlier CNN and FC layers. Action recovery, however, stays near-perfect ($\geq 99.6\%$) across all components, indicating that behavioral signals are distributed throughout the network.

\textbf{Objective Composition.} Removing the action loss (`w/o Act.') collapses action accuracy to near-random ($19.0\%$) without improving visual fidelity. Dropping LPIPS marginally raises PSNR to $18.94$ but inflates LPIPS to $0.401$, while removing $L_1$ degrades both PSNR ($17.90$) and LPIPS ($0.409$). These trade-offs highlight the need for a multi-objective formulation balancing pixel-level distortion and perceptual similarity.

%% file: tables/baselines.tex
\begin{table}[htbp]
    \centering
    \setlength{\tabcolsep}{2pt}
    \vspace{-10pt}
    \caption{Comparison of \sysname with baselines under PPO and A2C, \emph{Act. Acc.} denotes action accuracy and \emph{Infer.} denotes inference time per frame/image.}
    \label{tab:baselines}
    \small
    \resizebox{\linewidth}{!}{%
        \begin{tabular}{@{}lccccccgggggg@{}}
        \toprule
        & \multicolumn{6}{c}{\textbf{PPO}} & \multicolumn{6}{c}{\textbf{A2C}} \\
        \cmidrule(lr){2-7} \cmidrule(lr){8-13}
        \textbf{Method}
        & \textbf{MSE}$\downarrow$
        & \textbf{PSNR}$\uparrow$
        & \textbf{SSIM}$\uparrow$
        & \textbf{LPIPS}$\downarrow$
        & \textbf{Act. Acc.}$\uparrow$
        & \textbf{Infer.}
        & \textbf{MSE}$\downarrow$
        & \textbf{PSNR}$\uparrow$
        & \textbf{SSIM}$\uparrow$
        & \textbf{LPIPS}$\downarrow$
        & \textbf{Act. Acc.}$\uparrow$
        & \textbf{Infer.} \\
        \midrule
        DLG~\citep{zhu2019deep}
        & 0.289 $\pm$ 0.052 & 5.474 $\pm$ 0.826 & 0.043 $\pm$ 0.012 & 1.224 $\pm$ 0.096 & 16.6 & 31.0s
        & 0.259 $\pm$ 0.072 & 6.106 $\pm$ 1.544 & 0.036 $\pm$ 0.006 & 1.184 $\pm$ 0.127 & 51.0 & 76.1s \\
        IG~\citep{geiping2020inverting}
        & 0.161 $\pm$ 0.089 & 8.727 $\pm$ 2.791 & 0.166 $\pm$ 0.087 & 0.766 $\pm$ 0.110 & 31.5 & 39.3s
        & 0.114 $\pm$ 0.021 & 9.511 $\pm$ 0.810 & 0.093 $\pm$ 0.023 & 0.686 $\pm$ 0.038 & 32.5 & 39.6s \\
        LtI~\citep{wu2023learning}
        & 0.023 $\pm$ 0.009 & 16.790 $\pm$ 1.708 & 0.529 $\pm$ 0.097 & 0.671 $\pm$ 0.092 & \textbf{100.0} & 1.2ms
        & 0.018 $\pm$ 0.006 & 17.652 $\pm$ 1.366 & 0.508 $\pm$ 0.062 & 0.639 $\pm$ 0.071 & \textbf{100.0} & 0.2ms \\
        \midrule
        \textbf{\sysname}
        & \textbf{0.014} $\pm$ 0.007 & \textbf{18.771} $\pm$ 1.778 & \textbf{0.627} $\pm$ 0.105 & \textbf{0.362} $\pm$ 0.070 & \textbf{100.0} & 4.5ms
        & \textbf{0.014} $\pm$ 0.005 & \textbf{18.873} $\pm$ 1.391 & \textbf{0.583} $\pm$ 0.061 & \textbf{0.358} $\pm$ 0.060 & \textbf{100.0} & 3.0ms \\
        \bottomrule
        \end{tabular}%
    }
    \vspace{-4pt}
\end{table}

%% file: tables/ppo_adaptation.tex
\begin{wraptable}{r}{0.5\textwidth}
    \centering
    \setlength{\tabcolsep}{2pt}
    \vspace{-14pt}
    \caption{Few-shot adaptation on held-out scenes (PPO). See \Cref{sec:ext_fewshot_adaptation} for A2C results.}
    \vspace{-3pt}
    \label{tab:ppo_fewshot}
    \resizebox{\linewidth}{!}{%
        \begin{tabular}{@{}cccccc@{}}
            \toprule
            \hiderowcolors
            $p$ (\%) & \textbf{MSE}$\downarrow$ & \textbf{PSNR}$\uparrow$ & \textbf{SSIM}$\uparrow$ & \textbf{LPIPS}$\downarrow$ & \textbf{Act. Acc.}$\uparrow$ \\
            \midrule
            \showrowcolors
            0 & $0.014 {\scriptstyle \pm 0.007}$ & $18.8 {\scriptstyle \pm 1.8}$ & $0.627 {\scriptstyle \pm 0.105}$ & $0.362 {\scriptstyle \pm 0.070}$ & $100.0\%$ \\
            \hiderowcolors
            10 & $0.009 {\scriptstyle \pm 0.004}$ & $20.9 {\scriptstyle \pm 1.8}$ & $0.671 {\scriptstyle \pm 0.096}$ & $0.273 {\scriptstyle \pm 0.077}$ & $100.0\%$ \\
            \showrowcolors
            20 & $0.007 {\scriptstyle \pm 0.003}$ & $22.1 {\scriptstyle \pm 2.0}$ & $0.697 {\scriptstyle \pm 0.095}$ & $0.224 {\scriptstyle \pm 0.073}$ & $100.0\%$ \\
            \hiderowcolors
            30 & $0.006 {\scriptstyle \pm 0.003}$ & $23.0 {\scriptstyle \pm 2.1}$ & $0.717 {\scriptstyle \pm 0.092}$ & $0.186 {\scriptstyle \pm 0.063}$ & $100.0\%$ \\
            \showrowcolors
            40 & $0.005 {\scriptstyle \pm 0.003}$ & $23.4 {\scriptstyle \pm 2.2}$ & $0.726 {\scriptstyle \pm 0.095}$ & $0.175 {\scriptstyle \pm 0.064}$ & $100.0\%$ \\
            \hiderowcolors
            50 & $0.005 {\scriptstyle \pm 0.002}$ & $23.8 {\scriptstyle \pm 2.1}$ & $0.740 {\scriptstyle \pm 0.092}$ & $0.161 {\scriptstyle \pm 0.061}$ & $100.0\%$ \\
            \bottomrule
        \end{tabular}
    }
    \vspace{-10pt}
\end{wraptable}

%% file: tables/defense.tex
\begin{table}[htbp]
    \centering
    \setlength{\tabcolsep}{2pt}
    \vspace{-15pt}
    \caption{Evaluation of robustness of defenses under PPO and A2C.}
    \label{tab:defenses}
    \small
    \resizebox{\linewidth}{!}{%
        \begin{tabular}{@{}llcccccggggg@{}}
        \toprule
        & & \multicolumn{5}{c}{\textbf{PPO}} & \multicolumn{5}{c}{\textbf{A2C}} \\
        \cmidrule(lr){3-7} \cmidrule(lr){8-12}
        \textbf{Defense} & \textbf{Setting}
        & \textbf{MSE}$\downarrow$
        & \textbf{PSNR}$\uparrow$
        & \textbf{SSIM}$\uparrow$
        & \textbf{LPIPS}$\downarrow$
        & \textbf{Act. Acc.}$\uparrow$
        & \textbf{MSE}$\downarrow$
        & \textbf{PSNR}$\uparrow$
        & \textbf{SSIM}$\uparrow$
        & \textbf{LPIPS}$\downarrow$
        & \textbf{Act. Acc.}$\uparrow$ \\
        \midrule
        No Defense & --- & $0.018 {\scriptstyle \pm 0.014}$ & $18.9 {\scriptstyle \pm 3.7}$ & $0.628 {\scriptstyle \pm 0.143}$ & $0.374 {\scriptstyle \pm 0.090}$ & $99.9\%$ & $0.015 {\scriptstyle \pm 0.007}$ & $18.4 {\scriptstyle \pm 1.5}$ & $0.530 {\scriptstyle \pm 0.068}$ & $0.373 {\scriptstyle \pm 0.053}$ & $99.1\%$ \\
        \midrule
        \multirow{3}{*}{\makecell{Quantization\\(bits)}}
          & Mild (8-bit) & $0.020 {\scriptstyle \pm 0.017}$ & $18.5 {\scriptstyle \pm 3.9}$ & $0.619 {\scriptstyle \pm 0.148}$ & $0.384 {\scriptstyle \pm 0.095}$ & $99.9\%$ & $0.016 {\scriptstyle \pm 0.007}$ & $18.4 {\scriptstyle \pm 1.5}$ & $0.529 {\scriptstyle \pm 0.068}$ & $0.374 {\scriptstyle \pm 0.052}$ & $98.3\%$ \\
          & 4-bit & $0.075 {\scriptstyle \pm 0.056}$ & $12.4 {\scriptstyle \pm 3.2}$ & $0.422 {\scriptstyle \pm 0.127}$ & $0.645 {\scriptstyle \pm 0.104}$ & $37.9\%$ & $0.057 {\scriptstyle \pm 0.042}$ & $13.3 {\scriptstyle \pm 2.7}$ & $0.428 {\scriptstyle \pm 0.070}$ & $0.556 {\scriptstyle \pm 0.101}$ & $50.2\%$ \\
          & 2-bit & $0.091 {\scriptstyle \pm 0.063}$ & $11.3 {\scriptstyle \pm 2.8}$ & $0.380 {\scriptstyle \pm 0.093}$ & $0.689 {\scriptstyle \pm 0.094}$ & $19.4\%$ & $0.078 {\scriptstyle \pm 0.031}$ & $11.5 {\scriptstyle \pm 1.8}$ & $0.404 {\scriptstyle \pm 0.052}$ & $0.653 {\scriptstyle \pm 0.054}$ & $41.4\%$ \\
        \midrule
        \multirow{3}{*}{\makecell{Pruning\\(keep~\%)}}
          & 90\% & $0.018 {\scriptstyle \pm 0.014}$ & $18.9 {\scriptstyle \pm 3.7}$ & $0.628 {\scriptstyle \pm 0.143}$ & $0.374 {\scriptstyle \pm 0.090}$ & $99.9\%$ & $0.015 {\scriptstyle \pm 0.007}$ & $18.4 {\scriptstyle \pm 1.5}$ & $0.530 {\scriptstyle \pm 0.068}$ & $0.373 {\scriptstyle \pm 0.053}$ & $99.1\%$ \\
          & 50\% & $0.018 {\scriptstyle \pm 0.014}$ & $18.9 {\scriptstyle \pm 3.7}$ & $0.628 {\scriptstyle \pm 0.143}$ & $0.374 {\scriptstyle \pm 0.090}$ & $99.9\%$ & $0.015 {\scriptstyle \pm 0.007}$ & $18.4 {\scriptstyle \pm 1.5}$ & $0.530 {\scriptstyle \pm 0.068}$ & $0.373 {\scriptstyle \pm 0.053}$ & $99.1\%$ \\
          & 10\% & $0.018 {\scriptstyle \pm 0.015}$ & $18.8 {\scriptstyle \pm 3.8}$ & $0.624 {\scriptstyle \pm 0.146}$ & $0.371 {\scriptstyle \pm 0.091}$ & $99.9\%$ & $0.017 {\scriptstyle \pm 0.008}$ & $18.0 {\scriptstyle \pm 1.7}$ & $0.517 {\scriptstyle \pm 0.069}$ & $0.383 {\scriptstyle \pm 0.056}$ & $89.0\%$ \\
        \midrule
        \multirow{3}{*}{\makecell{Noise\\($\sigma$)}}
          & $0.001$ & $0.018 {\scriptstyle \pm 0.014}$ & $18.8 {\scriptstyle \pm 3.8}$ & $0.626 {\scriptstyle \pm 0.145}$ & $0.376 {\scriptstyle \pm 0.092}$ & $99.8\%$ & $0.015 {\scriptstyle \pm 0.007}$ & $18.4 {\scriptstyle \pm 1.5}$ & $0.530 {\scriptstyle \pm 0.068}$ & $0.373 {\scriptstyle \pm 0.053}$ & $98.8\%$ \\
          & $0.01$ & $0.023 {\scriptstyle \pm 0.022}$ & $18.1 {\scriptstyle \pm 4.1}$ & $0.599 {\scriptstyle \pm 0.160}$ & $0.401 {\scriptstyle \pm 0.109}$ & $94.9\%$ & $0.016 {\scriptstyle \pm 0.007}$ & $18.3 {\scriptstyle \pm 1.6}$ & $0.525 {\scriptstyle \pm 0.069}$ & $0.377 {\scriptstyle \pm 0.054}$ & $96.2\%$ \\
          & $0.1$ & $0.049 {\scriptstyle \pm 0.034}$ & $14.2 {\scriptstyle \pm 3.3}$ & $0.438 {\scriptstyle \pm 0.148}$ & $0.572 {\scriptstyle \pm 0.111}$ & $59.9\%$ & $0.026 {\scriptstyle \pm 0.016}$ & $16.5 {\scriptstyle \pm 2.3}$ & $0.452 {\scriptstyle \pm 0.092}$ & $0.440 {\scriptstyle \pm 0.081}$ & $64.1\%$ \\
        \midrule
        \multirow{3}{*}{\makecell{DP-SGD\\($\varepsilon, \delta$)}}
          & ($10$, $10^{-5}$) & $0.069 {\scriptstyle \pm 0.037}$ & $12.2 {\scriptstyle \pm 2.2}$ & $0.345 {\scriptstyle \pm 0.085}$ & $0.658 {\scriptstyle \pm 0.085}$ & $21.4\%$ & $0.050 {\scriptstyle \pm 0.018}$ & $13.3 {\scriptstyle \pm 1.5}$ & $0.322 {\scriptstyle \pm 0.053}$ & $0.563 {\scriptstyle \pm 0.068}$ & $27.4\%$ \\
          & ($5$, $10^{-5}$) & $0.069 {\scriptstyle \pm 0.035}$ & $12.1 {\scriptstyle \pm 2.1}$ & $0.342 {\scriptstyle \pm 0.084}$ & $0.663 {\scriptstyle \pm 0.082}$ & $21.4\%$ & $0.050 {\scriptstyle \pm 0.018}$ & $13.2 {\scriptstyle \pm 1.4}$ & $0.319 {\scriptstyle \pm 0.051}$ & $0.559 {\scriptstyle \pm 0.068}$ & $23.7\%$ \\
          & ($1$, $10^{-5}$) & $0.071 {\scriptstyle \pm 0.039}$ & $12.0 {\scriptstyle \pm 2.2}$ & $0.346 {\scriptstyle \pm 0.084}$ & $0.658 {\scriptstyle \pm 0.084}$ & $17.6\%$ & $0.049 {\scriptstyle \pm 0.017}$ & $13.3 {\scriptstyle \pm 1.4}$ & $0.322 {\scriptstyle \pm 0.053}$ & $0.562 {\scriptstyle \pm 0.067}$ & $24.5\%$ \\
        \bottomrule
        \end{tabular}%
    }
    \vspace{-6pt}
\end{table}

%% file: tables/ppo_ablation.tex
\begin{wraptable}{r}{0.55\textwidth}
    \centering
    \setlength{\tabcolsep}{2.5pt}
    \vspace{-12pt}
    \caption{Ablation study conducted on the held-out set (PPO). See \Cref{sec:ext_ablations} for A2C results.}
    \label{tab:ablations}
    \vspace{-4pt}
    \resizebox{\linewidth}{!}{%
        \begin{tabular}{@{}llccccc@{}}
            \toprule
             & \textbf{Settings} & \textbf{MSE}$\downarrow$ & \textbf{PSNR}$\uparrow$ & \textbf{SSIM}$\uparrow$ & \textbf{LPIPS}$\downarrow$ & \textbf{Act. Acc.}$\uparrow$ \\
            \midrule
            \multirow{6}{*}{\rotatebox{90}{$T$}}
                & 1 & $0.055 {\scriptstyle \pm 0.020}$ & $12.85 {\scriptstyle \pm 1.52}$ & $0.513 {\scriptstyle \pm 0.078}$ & $0.781 {\scriptstyle \pm 0.048}$ & 100.0\% \\
                & 8 & $0.014 {\scriptstyle \pm 0.006}$ & $18.74 {\scriptstyle \pm 1.73}$ & $0.617 {\scriptstyle \pm 0.108}$ & $0.375 {\scriptstyle \pm 0.071}$ & 100.0\% \\
                & 16 & $0.015 {\scriptstyle \pm 0.007}$ & $18.57 {\scriptstyle \pm 1.78}$ & $0.614 {\scriptstyle \pm 0.109}$ & $0.395 {\scriptstyle \pm 0.070}$ & 100.0\% \\
                & 32 & $0.014 {\scriptstyle \pm 0.010}$ & $19.62 {\scriptstyle \pm 3.08}$ & $0.634 {\scriptstyle \pm 0.131}$ & $0.377 {\scriptstyle \pm 0.078}$ & 99.8\% \\
                & 64 & $0.014 {\scriptstyle \pm 0.012}$ & $20.29 {\scriptstyle \pm 4.25}$ & $0.666 {\scriptstyle \pm 0.141}$ & $0.375 {\scriptstyle \pm 0.091}$ & 99.9\% \\
            \midrule
            \multirow{3}{*}{\rotatebox{90}{Comp.}}
                & \cellcolor{gray!10}{CNN} & \cellcolor{gray!10}{$0.015 {\scriptstyle \pm 0.006}$} & \cellcolor{gray!10}{$18.48 {\scriptstyle \pm 1.72}$} & \cellcolor{gray!10}{$0.584 {\scriptstyle \pm 0.118}$} & \cellcolor{gray!10}{$0.384 {\scriptstyle \pm 0.073}$} & \cellcolor{gray!10}{100.0\%} \\
                & \cellcolor{gray!10}{FC} & \cellcolor{gray!10}{$0.020 {\scriptstyle \pm 0.013}$} & \cellcolor{gray!10}{$17.72 {\scriptstyle \pm 2.17}$} & \cellcolor{gray!10}{$0.605 {\scriptstyle \pm 0.112}$} & \cellcolor{gray!10}{$0.388 {\scriptstyle \pm 0.074}$} & \cellcolor{gray!10}{100.0\%} \\
                & \cellcolor{gray!10}{Heads} & \cellcolor{gray!10}{$0.047 {\scriptstyle \pm 0.029}$} & \cellcolor{gray!10}{$14.12 {\scriptstyle \pm 2.71}$} & \cellcolor{gray!10}{$0.530 {\scriptstyle \pm 0.113}$} & \cellcolor{gray!10}{$0.511 {\scriptstyle \pm 0.088}$} & \cellcolor{gray!10}{99.6\%} \\
            \midrule
            \multirow{4}{*}{\rotatebox{90}{Loss (w/o)}}
                & Act. & $0.015 {\scriptstyle \pm 0.007}$ & $18.56 {\scriptstyle \pm 1.77}$ & $0.613 {\scriptstyle \pm 0.108}$ & $0.376 {\scriptstyle \pm 0.066}$ & 19.0\% \\
                & L1 & $0.018 {\scriptstyle \pm 0.009}$ & $17.90 {\scriptstyle \pm 1.88}$ & $0.608 {\scriptstyle \pm 0.111}$ & $0.409 {\scriptstyle \pm 0.074}$ & 100.0\% \\
                & LPIPS & $0.014 {\scriptstyle \pm 0.006}$ & $18.94 {\scriptstyle \pm 1.73}$ & $0.627 {\scriptstyle \pm 0.107}$ & $0.401 {\scriptstyle \pm 0.088}$ & 100.0\% \\
                & MSE & $0.015 {\scriptstyle \pm 0.007}$ & $18.56 {\scriptstyle \pm 1.75}$ & $0.612 {\scriptstyle \pm 0.110}$ & $0.379 {\scriptstyle \pm 0.070}$ & 100.0\% \\
            \bottomrule
        \end{tabular}%
    }
    \vspace{-10pt}%
\end{wraptable}

%% file: sections/related_work.tex
\section{Related Work}
\label{sec:related_work}

\vspace{-1pt}
\textbf{Learning-based inversion.}
Classical gradient inversion reconstructs each gradient through iterative optimization~\citep{zhu2019deep, geiping2020inverting,yin2021see}, incurring per-sample computational overhead that often ranges from minutes to hours. Learning-based attacks amortize this cost with feed-forward inversion models trained on auxiliary data: \citet{wu2023learning} demonstrated that such models can match or surpass optimization-based methods, even in the presence of gradient compression and perturbation defenses. Subsequent work on efficiency-focused attacks further narrowed this gap~\citep{huang2021evaluating,shi2023scale}. In contrast, \sysname replaces per-gradient optimization entirely with a causal transformer that autoregressively decodes the full trajectory, enabling near real-time inversion without any gradient-by-gradient search.

\vspace{-1pt}
\textbf{Gradient inversion attacks in distributed learning.}
Gradient inversion attacks in distributed or federated learning are well-studied
\citep{zhu2019deep,zhao2020idlg,geiping2020inverting,yin2021see,jeon2021gradient,hatamizadeh2022gradvit,lu2022april}. They may rely on a stronger assumption, wherein the server actively modifies either parameters or model architecture~\citep{pasquini2022eluding,wen2022fishing,zhao2024loki,fowl2021robbing}. More recent work extends these attacks to Vision Transformers~\citep{hatamizadeh2022gradvit,lu2022april}, scales them to larger architectures~\citep{shi2023scale}, and systematizes the attack landscape~\citep{carletti2025sok}. Unlike these single-sample attacks, \sysname targets temporally correlated gradient streams from embodied policy-learning agents, exploiting cross-step dependencies that prior work treats as i.i.d.\ noise.

\vspace{-1pt}
\textbf{Privacy risks in RL and embodied AI.}
Gradient inversion tailored to RL and embodied AI remains limited. \citet{li2024privacy} investigated privacy risks for household robots trained via federated RL. At the trajectory level, \citet{du2023orl} proposed ORL-AUDITOR, an offline deep-RL auditing mechanism that uses cumulative rewards as dataset fingerprints, showing that sequence decision data contains exploitable structure. Recent work also studies federated RL under transition-dynamics constraints~\citep{he2025gradient}. Building on these directions, \sysname exploits the temporal structure across consecutive frames of embodied agents to amortize trajectory-level inversion in near real-time.

\vspace{-1pt}
\textbf{Defenses.} 
Various defense strategies have also been proposed against gradient inversion, including representation-level protection~\citep{sun2021soteria}, gradient-leakage resilient FL~\citep{wei2021gradient}, auditing via generative gradient leakage~\citep{li2022auditing}, CENSOR via orthogonal subspace Bayesian sampling~\citep{zhang2025censor}, SVDefense via singular value decomposition~\citep{luo2025svdefense}, differentially private FL~\citep{geyer2017differentially}, and sparse communication~\citep{aji2017sparse}. While these defenses reduce leakage in static settings, their effectiveness against amortized temporal attacks on sequential trajectory data remains open.

%% file: sections/conclusion.tex
\section{Conclusion}
\label{sec:conclusion}

We introduced \sysname, an amortized temporal gradient inversion attack that reconstructs private observation-action trajectories from ordered policy-learning gradients. Our analysis integrates closed-form discrete-action recovery, an information-theoretic bound on the temporal reconstruction gain, and a local equivalence between A2C and PPO gradients. Across held-out embodied scenes, \sysname achieves $18.8$ dB PSNR with near-perfect action recovery in $3$--$4.5$ ms per reconstructed frame, improving over the learning-based baseline by ${\sim}2$ dB PSNR. Lightweight defenses such as pruning and mild quantization leave substantial leakage, motivating sequence-aware privacy mechanisms and future extensions to continuous-action and off-policy settings, shuffled minibatches, and multi-epoch model updates.

%% file: sections/acknowledgement.tex
\section*{Acknowledgment}
\label{sec:ack}

This work was supported in part by the US National Science Foundation under grants 2442382 and 2238635, the Office of Naval Research under grants N00014-24-1-2730 and N00014-24-1-2663, and the Army Research Office under grant W911NF-24-1-0155. We also thank the University of Kentucky Center for Computational Sciences and Information Technology Services Research Computing for their support and use of the Lipscomb Compute Cluster (LCC) and associated research computing resources.

%% file: sections/appendix.tex
\appendix
\section*{Appendix}

\input{appendices/broader_impact}
\input{appendices/notations}

\input{appendices/preliminaries}

\input{appendices/threat_model}
\input{appendices/theoretical_analysis_derivation}
\input{appendices/evaluation_metrics}

\input{appendices/add_experimental_details}

\input{appendices/limitations}

%% file: appendices/broader_impact.tex
\section{Broader Impact}
\label{sec:broader_impact}

In this work, we show that trajectories of private observations and actions can be reconstructed from shared policy gradients in the honest-but-curious setting. Our findings apply to advantage-weighted actor-critic algorithms studied in this paper: the temporal coherence of embodied RL trajectories renders them substantially more vulnerable to gradient inversion than isolated images. Lightweight per-step defenses such as pruning and mild quantization leave substantial cross-step leakage intact. In contrast, stronger perturbations such as DP-SGD, aggressive quantization, and large Gaussian noise degrade reconstruction quality but may also affect utility or communication efficiency. Receiving policy gradients from an embodied agent is thus nearly equivalent to receiving the raw sensor stream, and the amortized attack reconstructs an entire trajectory in milliseconds, making it inexpensive to deploy once trained.

These findings motivate sequence-aware privacy mechanisms that budget protection across an entire rollout rather than treating each gradient update independently. Promising directions include secure aggregation, differential privacy adapted to temporal gradient streams, and defenses that disrupt the cross-step correlations \sysname exploits. However, such mechanisms often incur computational overhead prohibitive for resource-constrained embodied deployments reliant on real-time training, so developing lightweight, sequence-aware defenses that balance privacy and efficiency remains an open challenge. To reduce misuse risk while preserving reproducibility, we do not release pretrained attacker checkpoints or captured trajectory datasets. By exposing the extent of this threat, we hope to accelerate research toward genuinely private collaborative training of embodied RL agents.

%% file: appendices/notations.tex
\section{Notations}
\label{sec:notations}

\Cref{tab:notations} summarizes the key mathematical notations used throughout this paper.

\input{tables/notations}

%% file: tables/notations.tex
\begin{table}[ht]
    \centering
    \newcommand{\twocol}[1]{\begin{tabular}[l]{@{}l@{}}#1\end{tabular}}
    \vspace{-8pt}
    \caption{Key notations used throughout the paper.}
    \label{tab:notations}
    \resizebox{1.0\linewidth}{!}{
    \begin{tabular}{ll@{\hskip 0.8cm}ll}
    \toprule
    \textbf{Symbol} & \multicolumn{1}{l}{\textbf{Description}} & \textbf{Symbol} & \multicolumn{1}{l}{\textbf{Description}} \\
    \midrule
    $t,\;T$ & Timestep index and trajectory length & $\theta,\;\theta_0$ & Victim parameters and shared checkpoint \\
    $\mathbf{s}_t,\;a_t$ & State observation and action at step $t$ & $\hat{A}_t,\;\hat{R}_t$ & Advantage and return targets \\
    $\pi_\theta,\;V_\theta$ & Policy and value functions & $r_t(\theta)$ & PPO importance ratio \\
    $\epsilon,\;c_v,\;c_e$ & PPO clip, value, and entropy coefficients & $\mathcal{L}^{\mathrm{ac}}_t$ & Per-step actor-critic objective \\
    $\mathbf{g}_t \in \mathbb{R}^d$ & Per-step communicated gradient & $\mathcal{G}_{1:T}$ & \twocol{Observed gradient sequence\\$(\mathbf{g}_1,\ldots,\mathbf{g}_T)$} \\
    $\phi_{\mathrm{enc}},\;\phi_{\mathrm{img}}$ & Gradient and image encoders & $\phi_{\mathrm{temp}},\;\phi_{\mathrm{dec}}$ & Temporal transformer and decoder \\
    $\mathbf{z}_t,\;\mathbf{e}_t,\;\tilde{\mathbf{z}}_t$ & \twocol{Gradient token, image-feedback\\token, and contextualized latent} & $D,\;H_I,\;W_I,\;K$ & \twocol{Latent dimension, image size,\\and action-space size} \\
    $\beta_n,\;\tau,\;R$ & \twocol{Scheduled-sampling probability,\\warmup duration, rollout horizon} & $\alpha_1,\ldots,\alpha_5$ & Training-loss weights \\
    $\mathbf{g}_t^{(\pi)},\mathbf{g}_t^{(v)},\mathbf{g}_t^{(e)}$ & \twocol{Policy, value, and entropy\\gradient components} & $\mathbf{W}_a,\;\mathbf{h}(s),\;d_h$ & \twocol{Policy-head weights, hidden\\representation, hidden width} \\
    $\mathbf{G}_a,\;\sigma_k,\;\widetilde{\sigma}_k$ & \twocol{Policy-head gradient, ideal and\\observed column sums} & $\Delta_t$ & Identifiability gap \\
    $\Delta I_t$ & Temporal information gain & $\mathrm{MSE}^*(\cdot)$ & Bayes-optimal reconstruction error \\
    $P(\cdot\mid \mathbf{s}_t,a_t)$ & MDP transition kernel & & \\
    \bottomrule
    \end{tabular}
    }
\end{table}

%% file: appendices/preliminaries.tex
\section{Preliminaries}
\label{sec:preliminaries}

\subsection{MDP and Trajectory Notation}
\label{mdp}
We model the victim agent's interaction with the environment as a Markov decision process (MDP) with state space $\mathcal{S}$, action space $\mathcal{A}$, transition kernel $P(\mathbf{s}_{t+1}\mid \mathbf{s}_t,a_t)$, and reward function $r(\mathbf{s}_t,a_t)$. At timestep $t$, the agent observes state $\mathbf{s}_t$, samples an action $a_t \sim \pi_\theta(\cdot \mid \mathbf{s}_t)$ from a parameterized policy $\pi_\theta$, receives reward $r_t$, and transitions to the next state $\mathbf{s}_{t+1}$. A trajectory segment of length $T$ is denoted by
\[
\tau_{1:T} = (\mathbf{s}_1,a_1,r_1,\dots,\mathbf{s}_T,a_T,r_T).
\]
In our setting, $\mathbf{s}_t$ denotes the private RGB observation at timestep $t$, and the attacker observes only the communicated gradients derived from such trajectory data.

\subsection{Advantage-Weighted Actor-Critic Objective}
\label{sec:actor_critic_preliminaries}

\sysname targets victims trained with advantage-weighted actor--critic algorithms; we instantiate our analysis and experiments with A2C~\citep{mnih2016asynchronous} and PPO~\citep{schulman2017proximal}. Both learn a stochastic policy $\pi_\theta(a\mid\mathbf{s})$ alongside a critic $V_\theta(\mathbf{s})$, and share the same credit signal, an advantage estimate $\hat{A}_t$ produced by Generalized Advantage Estimation~\citep{schulman2015high},
\[
\hat{A}_t = \sum_{k=0}^{\infty}(\gamma\lambda)^k\,\delta_{t+k},
\qquad
\delta_t = r_t + \gamma V(\mathbf{s}_{t+1}) - V(\mathbf{s}_t),
\]
with reward discount $\gamma\in(0,1)$, trace-decay $\lambda\in[0,1]$, and return target $\hat{R}_t = \hat{A}_t + V(\mathbf{s}_t)$. The two algorithms differ only in how the policy is updated.

\textbf{A2C}~\citep{mnih2016asynchronous} is a synchronous on-policy actor-critic that follows the advantage-weighted policy gradient
\[
\nabla_\theta J^{\mathrm{A2C}}(\theta) = \mathbb{E}_t\!\bigl[\nabla_\theta\log\pi_\theta(a_t\mid\mathbf{s}_t)\,\hat{A}_t\bigr],
\]
so the per-step policy loss is the REINFORCE-style~\citep{williams1992simple} estimator
\[
\ell^{\mathrm{pol,A2C}}_t(\theta) = -\hat{A}_t\,\log\pi_\theta(a_t\mid\mathbf{s}_t),
\]
identical in form to vanilla policy gradient but with $\hat{A}_t$ supplied by the learned critic rather than a Monte-Carlo return. Each rollout is used for one single gradient step and then discarded.

\textbf{PPO}~\citep{schulman2017proximal} retains the same actor-critic skeleton but replaces the log-likelihood estimator with a clipped importance-weighted surrogate that enables reuse of a rollout over multiple SGD epochs without the updated policy drifting far from the behavior policy. Let $\theta_0$ denote the broadcast policy under which the rollout was collected and define the importance ratio $r_t(\theta) = \pi_\theta(a_t\mid\mathbf{s}_t)/\pi_{\theta_0}(a_t\mid\mathbf{s}_t)$. PPO minimizes
\[
\ell^{\mathrm{pol,PPO}}_t(\theta) = -\min\!\bigl(r_t(\theta)\hat{A}_t,\,\mathrm{clip}(r_t(\theta),1-\epsilon,1+\epsilon)\hat{A}_t\bigr),
\]
with a small $\epsilon>0$ capping how far $r_t$ may move before the update is masked.

Let $\ell^{\mathrm{pol}}_t(\theta)\in\{\ell^{\mathrm{pol,A2C}}_t,\ell^{\mathrm{pol,PPO}}_t\}$ denote the algorithm-specific per-step policy loss. Both algorithms then optimize a common per-step objective
\[
\mathcal{L}^{\mathrm{ac}}_t(\theta) = \ell^{\mathrm{pol}}_t(\theta) + c_v\bigl(V_\theta(\mathbf{s}_t) - \hat{R}_t\bigr)^2 - c_e\,\mathcal{H}\!\bigl(\pi_\theta(\cdot\mid \mathbf{s}_t)\bigr),
\]
in which the critic regresses onto $\hat{R}_t$ (weighted by $c_v$), and the policy is regularized toward higher entropy (weighted by $c_e$). Because $r_t(\theta_0) = 1$, PPO's clipping is inactive in a neighborhood of $\theta_0$ and the two policy losses produce identical gradients at $\theta_0$; \Cref{prop:local_ppo_capture} formalizes this equivalence. Consequently, the local gradient decomposition and discrete-action recovery results apply uniformly to both algorithms, while the temporal-gain analysis later imposes additional assumptions on how the advantage and return estimators depend on the trajectory.

\subsection{Autoregressive Modeling}
\label{sec:autoregressive_modeling}

An autoregressive model factorizes a joint distribution over a sequence into a product of one-step conditional distributions:
\begin{equation}
p(x_{1:T}) = \prod_{t=1}^T p(x_t \mid x_{<t}),
\label{eq:ar_factorization}
\end{equation}
where $x_{<t} := (x_1,\dots,x_{t-1})$. Thus, the prediction at time $t$ is made using only past context.

In \sysname, the target sequence consists of per-timestep reconstructions $(\hat{\mathbf{s}}_t,\hat{a}_t)$, and autoregression arises through the visual context tokens fed into the causal transformer. Let $\mathbf{z}_t = \phi_{\mathrm{enc}}(\mathbf{g}_t)$ denote the encoded gradient at time $t$, and let $\mathbf{e}_t = \phi_{\mathrm{img}}(\mathbf{s}_t)$ for $t \geq 1$ denote a training-time encoded image-context token. The token $\mathbf{e}_0$ is a learned start token (a trainable parameter, not derived from any image). During the parallel teacher-forced warmup path, the transformer input is interleaved as
\[
\mathbf{X} = [\mathbf{z}_1, \mathbf{e}_0, \mathbf{z}_2, \mathbf{e}_1, \dots, \mathbf{z}_T, \mathbf{e}_{T-1}],
\]
so that the prediction at timestep $t$ is conditioned on the current gradient representation $\mathbf{z}_t$, together with the preceding gradient and lagged image-context tokens.

The lower-triangular causal mask $M_{ij} = \mathbf{1}[j \leq i]$ enforces the autoregressive constraint in the parallel path: the gradient token $\mathbf{z}_t$ at position $2t{-}1$ attends to $\{\mathbf{z}_1, \mathbf{e}_0, \ldots, \mathbf{z}_{t-1}, \mathbf{e}_{t-2}, \mathbf{z}_t\}$, receiving image context up to $\mathbf{e}_{t-2}$. The paired token $\mathbf{e}_{t-1}$ at position $2t$ is visible only to subsequent tokens $\mathbf{z}_{t+1}, \ldots$. Scheduled sampling and inference instead construct the context sequentially: after decoding timestep $t$, the encoded generated frame (or a sampled ground-truth frame during training) is appended as context for future steps.

%% file: appendices/theoretical_analysis_derivation.tex
\section{Proof and Derivations}
\label{sec:theoretical_analysis_derivation}

We provide complete proofs of the theoretical results stated in~\Cref{sec:theoretical_analysis}.

\subsection{Proof of Proposition~\ref{prop:local_ppo_capture} (A2C--PPO Equivalence and Decomposition)}
\label{sec:proof_local_form}

\begin{proof}
    We prove the result in three steps: (1) PPO's clipping is locally inactive at $\theta_0$; (2) the gradients of the A2C and PPO policy losses are identical at $\theta_0$; (3) aggregating these terms produces the decomposition.

    \textbf{Step 1 (clipping inactive).}
    By definition of the importance ratio (defined in \Cref{sec:problem_formulation}),
    \[
        r_t(\theta_0)
        = \frac{\pi_{\theta_0}(a_t \mid \mathbf{s}_t)}
              {\pi_{\theta_0}(a_t \mid \mathbf{s}_t)}
        = 1.
    \]
    Since $1$ lies in the interior of the clipping interval $[1-\epsilon,1+\epsilon]$ and $\theta \mapsto r_t(\theta)$ is continuous, there exists a neighborhood $\mathcal{U}_t$ of $\theta_0$ in which $r_t(\theta) \in [1-\epsilon,1+\epsilon]$ for all $\theta \in \mathcal{U}_t$.
    Within this neighborhood,
    \[
        \mathrm{clip}\bigl(r_t(\theta),\,1-\epsilon,\,1+\epsilon\bigr) = r_t(\theta),
    \]
    so the PPO loss reduces to
    \begin{equation}
        \label{eq:local_ppo_policy}
        \ell^{\mathrm{pol,PPO}}_t(\theta)
        = -\hat{A}_t\,r_t(\theta),\qquad \theta \in \mathcal{U}_t.
    \end{equation}

    \textbf{Step 2 (gradient equivalence at $\theta_0$).} Writing
    \[
        r_t(\theta)
        = \exp\!\Bigl(\log \pi_\theta(a_t \mid \mathbf{s}_t) - \log \pi_{\theta_0}(a_t \mid \mathbf{s}_t)\Bigr),
    \]
    where the denominator is constant with respect to $\theta$, and using $r_t(\theta_0) = 1$, the chain rule gives
    \[
        \nabla_\theta r_t(\theta)\big|_{\theta=\theta_0}
        = r_t(\theta_0) \cdot \nabla_\theta \log \pi_\theta(a_t \mid \mathbf{s}_t)\big|_{\theta=\theta_0}
        = \nabla_\theta \log \pi_\theta(a_t \mid \mathbf{s}_t)\big|_{\theta=\theta_0}.
    \]
    Therefore
    \begin{align*}
        \nabla_\theta \ell^{\mathrm{pol,PPO}}_t(\theta)\big|_{\theta_0}
        &= -\hat{A}_t\,\nabla_\theta r_t(\theta)\big|_{\theta_0}
        = -\hat{A}_t\,\nabla_\theta \log \pi_\theta(a_t\mid\mathbf{s}_t)\big|_{\theta_0} \\
        &= \nabla_\theta \bigl(-\hat{A}_t \log \pi_\theta(a_t\mid\mathbf{s}_t)\bigr)\big|_{\theta_0}
        = \nabla_\theta \ell^{\mathrm{pol,A2C}}_t(\theta)\big|_{\theta_0},
    \end{align*}
    establishing the A2C--PPO gradient equivalence at the broadcast checkpoint.

    \textbf{Step 3 (decomposition).}
    Since A2C and PPO share the same critic and entropy terms in \cref{eq:capture-loss} and produce the same policy-term gradient at $\theta_0$ by Step~2, the communicated gradient $\mathbf{g}_t = \nabla_\theta \mathcal{L}^{\mathrm{ac}}_t(\theta)\big|_{\theta_0}$ is identical in both algorithms and can be computed by differentiating either form. Using the A2C policy term for concreteness, we differentiate
    \begin{equation}
        \label{eq:local_loss}
        \mathcal{L}^{\mathrm{ac}}_t(\theta)
        = -\hat{A}_t\,\log\pi_\theta(a_t\mid\mathbf{s}_t)
        + c_v\bigl(V_\theta(\mathbf{s}_t) - \hat{R}_t\bigr)^2
        - c_e\,\mathcal{H}\!\bigl(\pi_\theta(\cdot \mid \mathbf{s}_t)\bigr).
    \end{equation}
    Collecting terms:
    \begin{align*}
        \mathbf{g}_t
        &= \nabla_\theta \mathcal{L}^{\mathrm{ac}}_t(\theta)\big|_{\theta=\theta_0} \\
        &= -\hat{A}_t \nabla_\theta \log \pi_\theta(a_t \mid \mathbf{s}_t)\big|_{\theta_0}
        + 2c_v\bigl(V_{\theta_0}(\mathbf{s}_t) - \hat{R}_t\bigr)\,\nabla_\theta V_\theta(\mathbf{s}_t)\big|_{\theta_0}
        - c_e\,\nabla_\theta \mathcal{H}\!\bigl(\pi_\theta(\cdot \mid \mathbf{s}_t)\bigr)\big|_{\theta_0},
    \end{align*}
    which is \Cref{eq:gradient_decomposition}.
\end{proof}

\subsection{Proof of Theorem~\ref{thm:action_identifiability} (Discrete-Action Recovery)}
\label{sec:proof_action_recovery}

We prove the two parts of \Cref{thm:action_identifiability} separately. Part (i) establishes the column-sum identity for the policy contribution; part (ii) bounds the entropy-induced perturbation and gives a sufficient condition under which \eqref{eq:action_recovery} remains exact.

\subsubsection*{Proof of part (i): exact recovery when $c_e=0$.}

\begin{proof}
    Fix a timestep $t$, and abbreviate $\mathbf{s} = \mathbf{s}_t$, $a = a_t$, and $\hat{A} = \hat{A}_t$. By \Cref{prop:local_ppo_capture}, the policy-term gradient at the broadcast checkpoint $\theta_0$ is
    \[
        \nabla_\theta\bigl(-\hat{A}\log\pi_\theta(a\mid\mathbf{s})\bigr)\big|_{\theta_0},
    \]
    for both A2C and PPO. Subsequent reasoning, therefore, applies to both algorithms. Under assumption (A1) the value head shares no parameters with $\mathbf{W}_a$, so $\nabla_{\mathbf{W}_a}(V_\theta(\mathbf{s})-\hat{R})^2 = \mathbf{0}$. When in addition $c_e=0$, the entropy term is absent, and the gradient with respect to the policy-head weights $\mathbf{W}_a$ is determined entirely by the policy-loss term.
    
    Let 
    \[
        \boldsymbol{\ell}=\mathbf{W}_a^\top\mathbf{h}(\mathbf{s})\in\mathbb{R}^K, \qquad \boldsymbol{\pi}(\mathbf{s})=\mathrm{softmax}(\boldsymbol{\ell}),
    \] 
    and let $\mathbf{e}_a \in \mathbb{R}^K$ denote the one-hot vector of the true action, i.e.  $[\mathbf{e}_a]_k = \mathbf{1}[k=a]$.
    Since
    \[
        \log\pi(a\mid \mathbf{s}) \;=\; \ell_a \;-\; \log\!\sum_{j=1}^{K}e^{\ell_j},
    \]
    differentiating with respect to $\ell_k$ gives the standard softmax identity
   \begin{equation}
        \label{eq:softmax-logit-grad}
        \frac{\partial \log\pi(a\mid \mathbf{s})}{\partial \ell_k}
        \;=\; \mathbf{1}[k=a] \;-\; \frac{e^{\ell_k}}{\sum_{j}e^{\ell_j}}
        \;=\; \mathbf{1}[k=a] - \pi(k\mid \mathbf{s}),
    \end{equation}
    Equivalently, in vector form, $\nabla_{\boldsymbol{\ell}}\log\pi(a\mid \mathbf{s}) = \mathbf{e}_a - \boldsymbol{\pi}(\mathbf{s})$ and consequently
    \[
        \nabla_{\boldsymbol{\ell}}\bigl(-\hat{A}\log\pi(a\mid \mathbf{s})\bigr)
        \;=\; \hat{A}\bigl(\boldsymbol{\pi}(\mathbf{s}) - \mathbf{e}_a\bigr).
    \]
    Next, since 
    \[
        \ell_k = \sum_{i=1}^{d_h}{[\mathbf{W}_a]_{i, k}h_i(\mathbf{s})},
    \] 
    we have
    \[
        \frac{\partial \ell_k}{\partial [\mathbf{W}_a]_{i,k'}}
        \;=\; h_i(\mathbf{s})\,\mathbf{1}[k=k'],
    \]
    Applying the chain rule,
    {
    \small
    \begin{equation}
        \label{eq:elementwise-Ga}
        \frac{\partial\bigl(-\hat{A}\log\pi(a\mid \mathbf{s})\bigr)}{\partial [\mathbf{W}_a]_{i,k}}
        \;=\; \sum_{k'=1}^{K}\frac{\partial\bigl(-\hat{A}\log\pi(a\mid \mathbf{s})\bigr)}{\partial \ell_{k'}}\cdot\frac{\partial \ell_{k'}}{\partial [\mathbf{W}_a]_{i,k}}
        \;=\; \hat{A}\,\bigl(\pi(k\mid \mathbf{s}) - \mathbf{1}[k=a]\bigr)\,h_i(\mathbf{s}).
    \end{equation}
    }
    Therefore
    \begin{equation}
        \label{eq:ppo-policy-grad}
        \mathbf{G}_a
        \;=\; \frac{\partial \bigl(-\hat{A} \log \pi(a \mid \mathbf{s})\bigr)}{\partial \mathbf{W}_a}
        \;=\; \hat{A} \cdot \mathbf{h}(\mathbf{s})\,\bigl(\boldsymbol{\pi}(\mathbf{s}) - \mathbf{e}_a\bigr)^\top,
    \end{equation}
    where $\boldsymbol{\pi}(\mathbf{s}) = [\pi(1 \mid \mathbf{s}),\ldots,\pi(K \mid \mathbf{s})]^\top$ and $\mathbf{e}_a$ is the one-hot vector of the true action.
    Hence
    \begin{equation}
        \label{eq:column-k}
        [\mathbf{G}_a]_{:,k} = \hat{A} \cdot \mathbf{h}(\mathbf{s})\,\bigl(\pi(k \mid \mathbf{s}) - \mathbf{1}[k = a]\bigr).
    \end{equation}
    
    Since $\mathbf{h}(\mathbf{s})$ is the output of a ReLU layer, $h_i(\mathbf{s}) \geq 0$ for all $i$, and $\lVert \mathbf{h}(\mathbf{s}) \rVert_1 > 0$ for any non-trivial input.
    Summing over the hidden dimension:
    \begin{equation}
        \label{eq:column-sum}
        \sigma_k
        \;:=\;
        \sum_{i=1}^{d_h} [\mathbf{G}_a]_{i,k}
        = \hat{A} \cdot \lVert \mathbf{h}(\mathbf{s}) \rVert_1 \,\bigl(\pi(k \mid \mathbf{s}) - \mathbf{1}[k = a]\bigr).
    \end{equation}
    Under (A4), $\pi(a\mid\mathbf{s}) < 1$, so $1-\pi(a\mid\mathbf{s})>0$ and $\pi(k\mid\mathbf{s}) > 0$ holds for at least one $k\neq a$ (in fact for all $k$ by full softmax support). If $\hat{A} > 0$, then
    $\sigma_a=-\hat{A}\|\mathbf{h}(\mathbf{s})\|_1(1-\pi(a\mid \mathbf{s}))<0$ and
    $\sigma_k=\hat{A}\|\mathbf{h}(\mathbf{s})\|_1\pi(k\mid \mathbf{s})\ge 0$ for $k\neq a$, with equality only if $\pi(k\mid\mathbf{s})=0$.
    Thus $a=\arg\min_k \sigma_k$ (uniquely, since $\sigma_a$ is strictly the smallest).
    If $\hat{A} < 0$, the signs reverse and $a=\arg\max_k \sigma_k$.
    This proves \cref{eq:action_recovery}.

    For $K \ge 3$, consider the sign-agnostic rule
    \[
        a^* = \operatorname*{arg\,max}_k |\sigma_k - \bar{\sigma}|,
        \qquad \bar{\sigma} = \tfrac{1}{K}\textstyle\sum_k \sigma_k,
    \]
    where $\bar{\sigma}=0$ by $\sum_k\pi(k\mid \mathbf{s})=1$.
    We have
    $|\sigma_a|=|\hat{A}|\|\mathbf{h}\|_1(1-\pi(a\mid \mathbf{s}))$ and
    $|\sigma_k|=|\hat{A}|\|\mathbf{h}\|_1\pi(k\mid \mathbf{s})$ for $k\neq a$.
    Because the softmax has full support, $\pi(j\mid \mathbf{s})>0$ for every $j$, so for any $k\neq a$ and $K\ge 3$ we have $\pi(k\mid \mathbf{s}) < 1-\pi(a\mid \mathbf{s})$ strictly (since $1-\pi(a\mid \mathbf{s})=\pi(k\mid \mathbf{s})+\sum_{j\neq a,k}\pi(j\mid \mathbf{s})$ and the residual sum is positive). Hence $|\sigma_k|<|\sigma_a|$ for all $k\neq a$, and the maximizer is unique at $a$.
    For $K=2$, the two magnitudes coincide, and the sign-aware rule is necessary.
\end{proof}

\subsubsection*{Proof of part (ii): robust recovery under entropy regularization.}

\begin{proof}
    Reusing the abbreviations $\mathbf{s}=\mathbf{s}_t$, $\hat{A}=\hat{A}_t$, $\boldsymbol{\ell}=\mathbf{W}_a^\top\mathbf{h}(\mathbf{s})$, and $\boldsymbol{\pi}=\mathrm{softmax}(\boldsymbol{\ell})$, we now retain the entropy term $-c_e\,\mathcal{H}(\boldsymbol{\pi})$ in $\mathcal{L}^{\mathrm{ac}}_t$ and quantify its contribution to the column sums.

    \textbf{Step 1 (entropy gradient with respect to logits).}
    Writing the entropy in a convenient form
    \[
        \mathcal{H}(\boldsymbol{\pi})
        \;=\; -\sum_{k=1}^{K}\pi_k\log\pi_k
        \;=\; \log Z - \sum_{k=1}^{K}\pi_k\,\ell_k,
        \qquad Z := \sum_{j=1}^{K} e^{\ell_j},
    \]
    and using the standard softmax identities $\partial \pi_j/\partial\ell_k = \pi_j(\delta_{jk}-\pi_k)$ and $\partial \log Z/\partial\ell_k = \pi_k$, we obtain
    \begin{equation}
        \label{eq:entropy_logit_grad}
        \frac{\partial \mathcal{H}}{\partial \ell_k}
        \;=\; \pi_k - \pi_k\bigl(\ell_k - \bar{\ell} + 1\bigr)
        \;=\; -\,\pi_k\,(\ell_k - \bar{\ell}),
        \qquad \bar{\ell} := \sum_{j=1}^{K} \pi_j\,\ell_j,
    \end{equation}
    or equivalently $\nabla_{\boldsymbol{\ell}}\mathcal{H} = -\,\boldsymbol{\pi}\odot(\boldsymbol{\ell} - \bar{\ell}\,\mathbf{1})$.

    \textbf{Step 2 (entropy gradient with respect to $\mathbf{W}_a$).}
    Composing \eqref{eq:entropy_logit_grad} with $\partial \ell_k/\partial[\mathbf{W}_a]_{i,k'} = h_i(\mathbf{s})\,\mathbf{1}[k=k']$ gives the closed-form expression
    \[
        \nabla_{\mathbf{W}_a}\mathcal{H}
        \;=\; -\,\mathbf{h}(\mathbf{s})\,\bigl(\boldsymbol{\pi}\odot(\boldsymbol{\ell}-\bar{\ell}\,\mathbf{1})\bigr)^\top.
    \]
    The contribution of the entropy term to column $k$ of the observed gradient is therefore
    \begin{equation}
        \label{eq:entropy_perturbation}
        \eta_k
        \;:=\; -c_e\sum_{i=1}^{d_h}\bigl[\nabla_{\mathbf{W}_a}\mathcal{H}\bigr]_{i,k}
        \;=\; c_e\,\lVert\mathbf{h}(\mathbf{s})\rVert_1\,\pi_k\,(\ell_k - \bar{\ell}),
    \end{equation}
    so that the observed column sums decompose as $\widetilde{\sigma}_k = \sigma_k + \eta_k$, where $\sigma_k$ is the policy-only column sum from part~(i).

    \textbf{Step 3 (perturbation bound).}
    Since $|\ell_k - \bar{\ell}| \le L_t$ and $\pi_k \le \pi_{\max}$, \eqref{eq:entropy_perturbation} gives the uniform bound
    \[
        |\eta_k| \;\le\; c_e\,\lVert\mathbf{h}(\mathbf{s})\rVert_1\,\pi_{\max}\,L_t
        \qquad \text{for every } k\in\{1,\ldots,K\}.
    \]

    \textbf{Step 4 (invariance of the argmin/argmax).}
    Suppose $\hat{A}>0$ (the case $\hat{A}<0$ is symmetric). By part~(i), $\sigma_{a_t} < \sigma_k$ for every $k\neq a_t$, with gap
    \[
        \sigma_k - \sigma_{a_t}
        \;=\; \hat{A}\,\lVert\mathbf{h}(\mathbf{s})\rVert_1\bigl(1 - \pi_{a_t} + \pi_k\bigr)
        \;\ge\; \Delta_t.
    \]
    The rule $\arg\min_k\widetilde{\sigma}_k$ still selects $a_t$ provided $\widetilde{\sigma}_{a_t} < \widetilde{\sigma}_k$ for every $k\neq a_t$, i.e.\
    \[
        \eta_k - \eta_{a_t} \;>\; -(\sigma_k - \sigma_{a_t}).
    \]
    A sufficient condition is $|\eta_k - \eta_{a_t}| < \sigma_k - \sigma_{a_t}$ for every $k\neq a_t$, which is in turn implied by the uniform bound $2\max_k|\eta_k| < \Delta_t$ (since $|\eta_k - \eta_{a_t}| \le 2\max_k|\eta_k|$ and $\sigma_k - \sigma_{a_t} \ge \Delta_t$). By Step~3, this uniform bound holds whenever $2\,c_e\,\lVert\mathbf{h}(\mathbf{s})\rVert_1\,\pi_{\max}\,L_t < \Delta_t$, i.e.\ \eqref{eq:entropy_safety}. The case $\hat{A}<0$ follows by reversing every inequality, completing the proof.
\end{proof}

\begin{corollary}
    \label{cor:uniform}
    Part~(i) of the recovery rule holds regardless of the policy's confidence.
    Even under a uniform policy $\pi_k = 1/K$ for all~$k$, the true-action column sum satisfies $|\sigma_{a_t}| = |\hat{A}|\,\lVert\mathbf{h}(\mathbf{s})\rVert_1\,(1 - 1/K) > 0$, while all other column sums satisfy $|\sigma_k| = |\hat{A}|\,\lVert\mathbf{h}(\mathbf{s})\rVert_1/K$.
    Therefore, for $K\ge 3$ the true action remains the unique maximizer under the sign-agnostic rule (since $(K-1)/K > 1/K$); for $K=2$ the magnitudes tie, and the sign-aware rule~\eqref{eq:action_recovery} is required, with the sign of $\hat A_t$ supplied by information outside the policy-head column magnitudes.
\end{corollary}

\begin{remark}[Scope of the entropy-safety condition]
    \label{rem:entropy_safety}
    Condition~\eqref{eq:entropy_safety} is a sufficient, not necessary, condition. The standard PPO/A2C entropy coefficients lie in $c_e\in[0.005,0.05]$, and the right-hand side $\Delta_t/\lVert\mathbf{h}(\mathbf{s})\rVert_1 = |\hat{A}_t|\,(1-\pi_{a_t}+\min_{k\neq a_t}\pi_k)$ remains bounded away from zero whenever $|\hat{A}_t|$ is non-negligible and the policy is not fully deterministic (\Cref{rem:deterministic_policy}). The left-hand side $2 c_e L_t \pi_{\max}$ scales with the logit spread $L_t$, which in practice is controlled by either (i) standard regularization of the policy head ($L_2$ weight decay or gradient clipping) or (ii) the entropy bonus itself, which actively penalizes large $L_t$. Empirically, the near-perfect action recovery reported in \Cref{tab:baselines} is consistent with the evaluated checkpoints remaining outside the degenerate regime, but we do not require this sufficient condition to be tight.
\end{remark}

\subsection{Conditional Sufficiency of $\mathbf{s}_t$ for $\mathbf{g}_t$ (\Cref{lem:gradient_sufficiency})}
\label{sec:proof_gradient_sufficiency}

We first restate the structural lemma that underlies the conditional-sufficiency argument in part~(ii) of \Cref{thm:temporal_gain}, with the assumptions on the advantage and return estimators stated explicitly. Throughout, we condition only on the broadcast checkpoint $\theta_0$ (no conditioning on realized advantages or returns).

\begin{lemma}[Conditional sufficiency of $\mathbf{s}_t$ for $\mathbf{g}_t$]
    \label{lem:gradient_sufficiency}
    Assume:
    \begin{itemize}
        \setlength\itemsep{0pt}
        \item[\emph{(M)}] \emph{Markov victim policy:} $a_t \sim \pi_{\theta_0}(\cdot\mid \mathbf{s}_t)$ depends on the trajectory only through $\mathbf{s}_t$.
        \item[\emph{(B)}] \emph{Causally adapted, un-normalized one-step estimators:} there exist deterministic functions $h_A,\,h_R$ such that
        \[
            \hat{A}_t \;=\; h_A(\mathbf{s}_t,a_t,\xi_t),
            \qquad
            \hat{R}_t \;=\; h_R(\mathbf{s}_t,a_t,\xi_t),
        \]
        where $\xi_t := (r_t,\mathbf{s}_{t+1})$ is the one-step Markov noise drawn from the reward kernel $R(\cdot\mid \mathbf{s}_t,a_t)$ and the transition kernel $P(\cdot\mid \mathbf{s}_t,a_t)$.
    \end{itemize}
    Then the per-step gradient is the deterministic map
    {
        \small
        \begin{align*}
            \small    
            \mathbf{g}_t \;&=\; G\!\bigl(\theta_0,\mathbf{s}_t,a_t,\xi_t\bigr) \\&:=\; \nabla_\theta\Bigl[
                    -h_A(\mathbf{s}_t,a_t,\xi_t)\,\log\pi_\theta(a_t\mid \mathbf{s}_t)
                + c_v\bigl(V_\theta(\mathbf{s}_t)-h_R(\mathbf{s}_t,a_t,\xi_t)\bigr)^2
                - c_e\,\mathcal{H}\!\bigl(\pi_\theta(\cdot\mid \mathbf{s}_t)\bigr)
                \Bigr]\Big|_{\theta=\theta_0}.
        \end{align*}
    }
    Consequently, conditional on $\theta_0$,
    \begin{equation}
        \label{eq:lemma_strong}
        \mathbf{g}_t \;\perp\!\!\!\perp\; \bigl(\{(\mathbf{s}_{t'},a_{t'})\}_{t'<t},\;\mathbf{g}_{1:t-1}\bigr)
        \;\bigm|\; (\mathbf{s}_t,a_t),
    \end{equation}
    and, marginalizing $a_t \sim \pi_{\theta_0}(\cdot\mid \mathbf{s}_t)$,
    \begin{equation}
        \label{eq:lemma_marginal}
        \mathbf{g}_t \;\perp\!\!\!\perp\; \mathbf{g}_{1:t-1} \;\bigm|\; \mathbf{s}_t.
    \end{equation}
\end{lemma}

\begin{proof}
    By \Cref{prop:local_ppo_capture} the A2C and PPO policy gradients agree at $\theta_0$, so $\mathbf{g}_t$ equals the displayed expression. Substituting the assumed forms of $\hat{A}_t,\hat{R}_t$ from~(B) shows that $\mathbf{g}_t = G(\theta_0,\mathbf{s}_t,a_t,\xi_t)$ depends on the trajectory only through $(\mathbf{s}_t,a_t,\xi_t)$. By the Markov property of the MDP, the noise $\xi_t = (r_t,\mathbf{s}_{t+1})$ is, conditional on $(\mathbf{s}_t,a_t)$, drawn from $R(\cdot\mid \mathbf{s}_t,a_t)\otimes P(\cdot\mid \mathbf{s}_t,a_t)$ independently of the past $\{(\mathbf{s}_{t'},a_{t'})\}_{t'<t}$ and of $\mathbf{g}_{1:t-1}$. Combined with the determinism of $G$, this yields \eqref{eq:lemma_strong}.

    To derive \eqref{eq:lemma_marginal}, write
    \[
        p(\mathbf{g}_t\mid \mathbf{s}_t,\mathbf{g}_{1:t-1},\theta_0)
        \;=\; \sum_{a_t \in \mathcal{A}} p(\mathbf{g}_t\mid \mathbf{s}_t,a_t,\mathbf{g}_{1:t-1},\theta_0)\;\pi_{\theta_0}(a_t\mid \mathbf{s}_t).
    \]
    The first factor equals $p(\mathbf{g}_t\mid \mathbf{s}_t,a_t,\theta_0)$ by~\eqref{eq:lemma_strong}; the second factor is $\pi_{\theta_0}(a_t\mid \mathbf{s}_t)$ by~(M), since $a_t$ is sampled from $\pi_{\theta_0}(\cdot\mid \mathbf{s}_t)$ independently of $\mathbf{g}_{1:t-1}$ given $\mathbf{s}_t$. Both factors are independent of $\mathbf{g}_{1:t-1}$, proving~\eqref{eq:lemma_marginal}.
\end{proof}

\begin{remark}[Scope of assumption (B)]
    \label{rem:scope_assumption_B}
    Assumption~(B) is satisfied by TD($0$) advantage--return pairs computed with stop-gradient on the bootstrap value (e.g.\ $\hat{A}_t = r_t + \gamma V_{\theta_0}(\mathbf{s}_{t+1}) - V_{\theta_0}(\mathbf{s}_t)$, $\hat{R}_t = r_t + \gamma V_{\theta_0}(\mathbf{s}_{t+1})$). It is \emph{violated} by GAE with $\lambda > 0$, which uses future $\xi_{t+k}$ for $k>0$, and by rollout-level normalization $\hat{A}_t \leftarrow (\hat{A}_t-\mu)/\sigma$, which couples every timestep through $(\mu,\sigma)$. We discuss the consequences of relaxing~(B) in \Cref{rem:diminishing_returns}.
\end{remark}

\subsection{Proof of Theorem~\ref{thm:temporal_gain} (Temporal Reconstruction Gain)}
\label{sec:proof_temporal}

\begin{proof}
    For part~(i), $\Delta I_t = I(\mathbf{s}_t;\mathbf{g}_{1:t-1}\mid \mathbf{g}_t)\ge 0$ by non-negativity of conditional mutual information. For the MSE statement, we work coordinate-wise. Writing $\mathbf{s}_t \in \mathbb{R}^{D_s}$ and using $\mathrm{MSE}^*(Y) = \mathbb{E}\!\left[\operatorname{tr}\mathrm{Cov}(\mathbf{s}_t\mid Y)\right] = \sum_{j=1}^{D_s} \mathbb{E}\!\left[\mathrm{Var}([\mathbf{s}_t]_j\mid Y)\right]$, the scalar law of total variance applied to each coordinate gives
    \[
        \mathrm{Var}([\mathbf{s}_t]_j \mid \mathbf{g}_t)
        \;=\; \mathbb{E}\!\bigl[\mathrm{Var}([\mathbf{s}_t]_j \mid \mathbf{g}_{1:t}) \bigm| \mathbf{g}_t\bigr]
        + \mathrm{Var}\!\bigl(\mathbb{E}[[\mathbf{s}_t]_j \mid \mathbf{g}_{1:t}] \bigm| \mathbf{g}_t\bigr)
        \;\ge\; \mathbb{E}\!\bigl[\mathrm{Var}([\mathbf{s}_t]_j \mid \mathbf{g}_{1:t}) \bigm| \mathbf{g}_t\bigr].
    \]
    Taking expectations and summing over $j$ yields $\mathrm{MSE}^*(\mathbf{g}_t) \ge \mathrm{MSE}^*(\mathbf{g}_{1:t})$.

    For part~(ii), assume~(M) and~(B) of \Cref{lem:gradient_sufficiency} and condition throughout on $\theta_0$. By~\eqref{eq:lemma_marginal}, $\mathbf{g}_t \perp\!\!\!\perp \mathbf{g}_{1:t-1} \mid \mathbf{s}_t$, equivalently $I(\mathbf{g}_t;\mathbf{g}_{1:t-1}\mid \mathbf{s}_t)=0$. This is exactly the Markov chain
    \[
        \mathbf{g}_{1:t-1}\;\to\;\mathbf{s}_t\;\to\;\mathbf{g}_t.
    \]
    Using the identity $I(X;Y\mid Z) = I(X;Y) - I(Y;Z) + I(Y;Z\mid X)$ with $(X,Y,Z) = (\mathbf{s}_t,\mathbf{g}_{1:t-1},\mathbf{g}_t)$ and $I(\mathbf{g}_t;\mathbf{g}_{1:t-1}\mid \mathbf{s}_t) = 0$ yields
    \[
        \Delta I_t
        \;=\;
        I(\mathbf{s}_t;\mathbf{g}_{1:t-1}\mid \mathbf{g}_t)
        \;=\;
        I(\mathbf{s}_t;\mathbf{g}_{1:t-1})
        -
        I(\mathbf{g}_{1:t-1};\mathbf{g}_t),
    \]
    which is \eqref{eq:temporal_identity}. Since mutual information is non-negative, $I(\mathbf{g}_{1:t-1};\mathbf{g}_t)\ge 0$, and therefore
    \begin{equation}
        \label{eq:temporal_upper_bound}
        \Delta I_t \;\le\; I(\mathbf{s}_t;\mathbf{g}_{1:t-1}),
    \end{equation}
    proving~\eqref{eq:temporal_upper_bound}.
\end{proof}

\begin{remark}[Beyond TD($0$): GAE and rollout normalization]
    \label{rem:diminishing_returns}
    The identity~\eqref{eq:temporal_identity} and bound~\eqref{eq:temporal_upper_bound} are exact under (M)+(B); the following discussion of the practical PPO/A2C setting is intuition, not a proved correction term. (i)~\emph{GAE with $\lambda{>}0$:} $\hat{A}_t=\sum_{k\ge 0}(\gamma\lambda)^k\,\delta_{t+k}$ depends on future noise $\xi_{t:T}$, so~(B) fails; since the policy gradient at $\theta_0$ scales linearly in $\hat{A}_t$ (\Cref{prop:local_ppo_capture}), the influence of $\xi_{t+k}$ on $\mathbf{g}_t$ is geometrically suppressed by $(\gamma\lambda)^k$~\citep{schulman2015high}. (ii)~\emph{Rollout normalization} $\hat{A}_t\!\leftarrow\!(\hat{A}_t-\mu)/\sigma$ couples all timesteps through two scalar statistics, breaking the per-step factorization of~(B) but only through a low-dimensional channel.

    The ablations in \Cref{tab:ablations,tab:ablations_a2c} show different empirical scaling across the two victim algorithms. Both practical settings use GAE and rollout-level advantage normalization, so assumption~(B) should be viewed as an idealized reference point rather than an exact description of these runs. A2C plateaus in $[18.3,18.8]$\,dB PSNR for $T\in\{8,16,32,64\}$, whereas PPO continues to improve from $18.74$\,dB at $T{=}8$ to $20.29$\,dB at $T{=}64$. We interpret this as algorithm- and checkpoint-dependent temporal signal strength rather than a theorem-level consequence of the TD(0) idealization.
\end{remark}

\subsection{Derivation of Training Objective}
\label{sec:derivation_of_training_objective}

Following \Cref{eq:total-loss}, let $\{\mathbf{s}_t\}_{t=1}^T$ denote a ground-truth trajectory and $\{\hat{\mathbf{s}}_t\}_{t=1}^T$ its reconstruction. We begin with two standard pixel-space reconstruction losses, namely the mean-squared error and the $\ell_1$ error:
{
\small
\begin{equation}
    \mathcal{L}_{\mathrm{MSE}} = \frac{1}{T} \sum_{t=1}^{T} \left\| \hat{\mathbf{s}}_t-\mathbf{s}_t \right\|_2^2, \qquad 
    \mathcal{L}_{L_1} = \frac{1}{T} \sum_{t=1}^{T} \left\| \hat{\mathbf{s}}_t-\mathbf{s}_t \right\|_1.
\end{equation}
}

To supervise action recovery, we include a cross-entropy loss on the predicted action distribution:
{
\small
\begin{equation}
    \mathcal{L}_{\mathrm{act}} = -\frac{1}{T} \sum_{t=1}^{T} \log \hat{p}_t(a_t),
\end{equation}
}%
where $\hat{p}_t(a_t)$ denotes the probability assigned to the ground-truth action $a_t$ at timestep $t$.

Finally, to improve perceptual quality beyond pixel-wise agreement, we further incorporate the LPIPS loss~\citep{zhang2018unreasonable}:
{
\small
\begin{equation}
    \mathcal{L}_{\mathrm{perc}} = \frac{1}{T} \sum_{t=1}^{T} \mathrm{LPIPS}\!\left(\hat{\mathbf{s}}_t,\mathbf{s}_t\right).
\end{equation}
}

To mitigate exposure bias arising from teacher forcing (see \Cref{sec:autoregressive-training}), we add a rollout loss $\mathcal{L}_{\mathrm{rollout}}$ that supervises a short autoregressive rollout under self-generated context.
Let $\ell(\hat{\mathbf{s}}, \mathbf{s}, \hat{p}, a) = \alpha_1\|\hat{\mathbf{s}} - \mathbf{s}\|_2^2 + \alpha_2\|\hat{\mathbf{s}} - \mathbf{s}\|_1 + \alpha_3(-\log \hat{p}(a)) + \alpha_4\,\mathrm{LPIPS}(\hat{\mathbf{s}}, \mathbf{s})$ denote the per-step loss.
We uniformly sample a start index $t_0 \sim \mathrm{Uniform}\{1, 2, \ldots, T{-}R{+}1\}$ and unroll $R$ steps with predicted frames fed back as image context:
{
\small
\begin{equation}
    \mathcal{L}_{\mathrm{rollout}} = \frac{1}{R} \sum_{r=0}^{R-1} \ell\!\left(\hat{\mathbf{s}}_{t_0+r},\, \mathbf{s}_{t_0+r},\, \hat{p}_{t_0+r},\, a_{t_0+r}\right),
\end{equation}
}%
where $\hat{\mathbf{s}}_{t_0+r} = \phi_{\mathrm{dec}}(\tilde{\mathbf{z}}_{t_0+r})$ uses the previous predicted frame as image context without backpropagating gradients through it.
This prevents the model from optimizing the quality of its own context inputs, ensuring the rollout trains robustness to imperfect context rather than learning to produce easily-decoded outputs.
We set $R{=}3$ in all experiments.

The overall objective combines these terms to jointly enforce reconstruction fidelity, perceptual realism, accurate action prediction, and robustness to autoregressive error accumulation:
\begin{equation}
\mathcal{L}_{\mathrm{total}} = \alpha_1 \mathcal{L}_{\mathrm{MSE}} + \alpha_2 \mathcal{L}_{\mathrm{L1}} + \alpha_3 \mathcal{L}_{\mathrm{act}} + \alpha_4 \mathcal{L}_{\mathrm{perc}} + \alpha_5 \mathcal{L}_{\mathrm{rollout}}.
\end{equation}

%% file: appendices/evaluation_metrics.tex
\section{Evaluation Metrics}
\label{sec:evaluation_metrics}

We present the evaluation metrics used to assess the correctness of the reconstructed images. 

\textbf{Mean Squared Error (MSE$\downarrow$)~\citep{wang2009mean}.} It computes the average squared differences between corresponding pixels in the original and reconstructed images. As a pixel-wise distortion metric, lower values indicate better reconstruction. 

\textbf{Peak Signal to Noise Ratio (PSNR$\uparrow$)~\citep{lin2005visual}.}
It is a logarithmic transformation of the MSE in decibels, offering an intuitive interpretation where higher values indicate superior image fidelity. 

\textbf{Structural Similarity (SSIM$\uparrow$)~\citep{wang2004image}.} It evaluates perceptual image quality by comparing luminance, contrast, and structural components between original and reconstructed images, capturing inter-pixel dependencies better than error-based metrics like MSE. 

\textbf{Learned Perceptual Image Patch Similarity (LPIPS$\downarrow$).~\citep{zhang2018unreasonable}}
It measures perceptual distance using deep features from a pre-trained convolutional neural network and correlates better with human perceptual judgments than traditional metrics. Lower LPIPS values indicate higher perceptual similarity.

\textbf{Action Accuracy $\uparrow$.}
It reports the percentage of correctly inferred actions from observed gradients, measuring the attacker's ability to recover the victim's navigation decisions. Higher accuracy indicates more severe leakage of private control information.

%% file: appendices/add_experimental_details.tex
\section{Additional Experiment Details}
\label{sec:add_experimental_details}

In this section, we provide additional details on the experimental setup and training procedure described in \Cref{sec:evaluation}.

\textbf{Environment and Navigation Task.} We use the AI2-THOR simulator~\citep{kolve2017ai2} with a point-goal navigation task. At the start of each episode, the agent and the target are initialized at randomly selected navigable locations. The objective is to navigate to within a $1.0$-meter radius of the target location within a maximum of $200$ time steps. A step penalty of $-0.01$ is applied at each time step, and a sparse terminal reward of $+1.0$ is granted upon successfully reaching the target. To prevent overfitting to specific visual configurations, the training procedures interleave multiple indoor scenes and randomly shuffle the environment at the end of each episode.

\textbf{Victim Model.} The default victim follows a standard actor-critic architecture~\citep{konda1999actor}: a shared convolutional encoder feeds parallel policy and value heads. The encoder is a four-block stack with channel widths $(32, 64, 64, 64)$, kernel sizes $(8, 4, 3, 3)$, and strides $(4, 2, 1, 2)$, and paddings $(2, 1, 1, 1)$, using \texttt{ReLU} activations and batch normalization. For an $84{\times}84$ RGB input, these settings produce a $5{\times}5{\times}64=1600$ flattened feature map. The flattened features are passed through a linear projection to a $512$-dimensional hidden representation (followed by \texttt{ReLU}) that is shared by two linear heads: a policy head over $|\mathcal{A}|{=}5$ discrete actions and a scalar value head. Including convolution, batch-normalization, fully connected, policy-head, and value-head parameters, the resulting model has a gradient dimension of $d = 936{,}102$, which defines the attacker's input space.

\textbf{Attacker (\sysname) Model.} 
\label{sec:attacker_impl}
The attacker model uses the autoregressive architecture described in \Cref{sec:overview}. The gradient encoder applies a learned input projection $W_{\mathrm{proj}}\!\in\!\mathbb{R}^{d \times d_{\mathrm{enc}}}$ followed by residual blocks and a linear map to $\mathbb{R}^D$. The temporal transformer uses $N{=}6$ layers of pre-normalized multi-head causal self-attention with $8$ heads and FlashAttention~\citep{dao2022flashattention}. By default, the transformer applies learned absolute positional embeddings $\mathbf{P}\!\in\!\mathbb{R}^{1\times 2T\times D}$ to the input sequence. For longer sequences ($T{=}64$), where the fixed embedding table becomes a bottleneck, we switch to Rotary Position Embeddings (RoPE)~\citep{su2024roformer}. The residual decoder upsamples the contextualized latent to $84{\times}84$ RGB images and jointly predicts a discrete action distribution. In the standard configuration ($d_{\mathrm{enc}}{=}2048$, $D{=}1024$), the model has approximately $2.27$ billion parameters, dominated by the input projection ($d \cdot d_{\mathrm{enc}} \approx 1.9 \times 10^9$ weights). In the ablation study, we use $d_{\mathrm{enc}}{=}1024$, reducing the model to approximately $1.21$ billion parameters.

\textbf{Dataset and Augmentation.} We prepare the training dataset by running the victim policy in AI2-THOR~\citep{kolve2017ai2} and capturing per-step gradients alongside the corresponding RGB observations and actions. The data is stored in \texttt{HDF5}~\citep{hdf5} format as temporally ordered trajectories. We use $108$ out of $120$ available indoor scenes for training and reserve the remaining $12$ for held-out test scenes. To enhance dataset diversity, we apply color-jitter transformations---specifically perturbations in brightness, contrast, saturation, and hue---to the recorded observations. Subsequently, gradients are recomputed through the frozen victim model using the exact training loss objective, which may be either the PPO surrogate or the A2C policy gradient incorporating GAE advantages, value loss, and entropy.

\textbf{Computational Requirements.} We implement \sysname in Python using PyTorch~\citep{paszke2019pytorch} as the primary framework for both the victim and attacker models. Default attacker training is distributed across six NVIDIA H200 GPUs using PyTorch DistributedDataParallel (DDP); the additional victim experiments use four or six GPUs, as reported in \Cref{tab:victim_variant_setup}. These GPUs are not required for the method itself: comparable accelerator clusters such as NVIDIA A100 GPUs are sufficient, with lower-end GPUs mainly increasing training time.

\textbf{Hyperparameter Settings.} \Cref{tab:hyperparameters} summarizes the main hyperparameters used for the default \sysname model training configuration. The complete configuration used for experiments is provided in the accompanying \texttt{yaml} file with the source code.

\input{tables/hyperparameters}

\subsection{Runtime of Experiments}
\label{sec:runtime}

\Cref{tab:runtime} reports end-to-end training runtimes for the main \sysname configurations. The runtimes of the additional victim configurations are reported in \Cref{tab:victim_variant_setup} because these experiments use slightly different, heterogeneous GPU configurations.

\input{tables/runtime}

\subsection{Few-Shot Adaptation}
\label{sec:ext_fewshot_adaptation}

\input{tables/a2c_adaptation}

We extend the few-shot adaptation analysis from \Cref{sec:few_shot_adaptation} to the A2C victim policy. \Cref{tab:a2c_fewshot} reports the same adaptation sweep ($p \in \{0, 10, 20, 30, 40, 50\}\%$) on held-out scenes. The trends mirror the PPO setting: zero-shot transfer already achieves near-perfect action recovery and reasonable visual fidelity (PSNR $= 18.9$), while adapting with $50\%$ of target trajectories raises PSNR to $25.0$ and lowers LPIPS to $0.132$.

\subsection{Applicability to Victim Configurations}
\label{sec:victim_variants}

To assess the broader applicability of \sysname, we examine trajectory leakage beyond the default shared-CNN victim with RGB observations and five actions. Alternative visual encoders and a larger residual backbone vary the representation and capacity of the victim, while recurrent state makes current decisions depend on observation history. Multi-modal inputs introduce goal information alongside visual observations. With larger discrete action spaces, we assess whether accurate action recovery persists as the number of candidate actions increases.

The PPO variants include a compact Vision Transformer (Tiny ViT)~\citep{dosovitskiy2020image}, two IMPALA-style residual CNNs~\citep{espeholt2018impala}, a CNN with a gated recurrent unit (GRU), a multi-modal CNN, and the original CNN with ten actions. Here, multi-modal inputs comprise RGB observations and an egocentric point-goal vector encoding distance and relative bearing. For each victim, we train a separate \sysname attacker using unaugmented captures rather than the augmented data used to train the base models. The sequences used for each training run are listed in \Cref{tab:victim_variant_setup}. All other evaluation settings are kept similar to the base settings. These experiments assess applicability across matched victim settings in AI2-THOR point-goal navigation, rather than the transferability of a single attacker across architectures.

\textbf{Architectures and observations.} Tiny ViT uses $14{\times}14$ patches, $128$-dimensional embeddings, two transformer layers with four attention heads and MLP width $512$, and a class token feeding the policy and value heads. It is a compact ViT implementation, not a pretrained model. The IMPALA-style CNN has three stages with channel widths $(16,32,32)$, each comprising a convolution, max pooling, two residual blocks, and group normalization, followed by a $256$-unit fully connected layer. The wider variant doubles the channel widths to $(32,64,64)$, increasing the victim from $1.09$M to $2.37$M parameters. CNN+GRU retains the default convolutional stack but uses a $256$-unit projection and GRU state. The multi-modal victim fuses the default $512$-dimensional visual features with a two-layer, $128$-unit goal encoder, then projects to $512$ shared features. Its three goal inputs are distance divided by $10$ meters and clipped at one, and the sine and cosine of the egocentric goal bearing. The attacker receives gradients only and reconstructs RGB and actions, not the goal vector.

\textbf{Actions and recurrent capture.} The ten-action set adds \texttt{MoveBack}, \texttt{MoveLeft}, \texttt{MoveRight}, \texttt{Crouch}, and \texttt{Stand} to the default five actions; rotation and look angles remain $15^\circ$. For CNN+GRU, hidden state is carried through each episode and reset at episode boundaries. Per-step PPO gradients use the recorded pre-step hidden state as a detached input, without backpropagation through earlier observations. The hidden state is not an attacker input.

\textbf{Attacker training and cost.} 
As in \Cref{tab:hyperparameters}, each attacker uses $T{=}8$ and trains for $50$ epochs with the default transformer, residual decoder, and loss weights. 
We adapt the default attacker's input dimension and action-head size to match each victim. For the wider IMPALA victim, we also replace the default single dense input projection with two learned projections, using an intermediate width of $768$ and an output width of $2048$, to reduce memory use. All runs retain the batch size of $8$ per GPU and two accumulation steps in \Cref{tab:hyperparameters}.
\Cref{tab:victim_variant_setup} reports training sequences, model sizes, GPU counts, and elapsed time through the last training/validation record. These times exclude victim training and data generation. Thus, neither compute nor attacker capacity is held fixed across variants, and the offline cost remains substantial despite fast online inference.

\input{tables/victim_variant_setup}

\textbf{Reconstruction results.} \Cref{tab:victim_variants} shows $99.5\%$ action accuracy with ten actions. The recurrent and multi-modal victims yield $18.62$ and $19.78$\,dB PSNR, respectively; only RGB and actions are reconstructed in the multi-modal setting. The wider residual victim has $2.37$M parameters and yields $22.36$\,dB PSNR. Tiny ViT retains $99.0\%$ action accuracy but reaches only $13.99$\,dB PSNR, so reliable action recovery does not imply equally reliable image recovery. Since trajectories, dataset sizes, and attacker capacities differ across victims, differences in reconstruction fidelity cannot be attributed solely to the victim architecture.

\input{tables/victim_variants}

\subsection{Reconstruction under Gradient Aggregation}
\label{sec:aggregation}

\sysname's main evaluation assumes access to ordered per-step gradients. When a learner shares only an aggregate across several steps, the server loses the direct correspondence between individual gradients and observation--action pairs. We therefore examine whether visual content and action proportions can still be reconstructed at the window level, comparing an unadapted per-step attacker, fine-tuning on matched aggregates, and training from scratch on aggregates.

Using the default five-action PPO victim, we average groups of four or eight consecutive stored gradients within each episode, preserve the order of the resulting aggregates, and discard incomplete groups at episode boundaries. Each aggregate is paired with the last RGB frame in its group and the normalized histogram of the actions in that group. The attacker receives $T{=}8$ aggregated gradients per sequence in both settings. Each aggregate corresponds to four original steps in Agg4 or eight in Agg8, so the sequences span $32$ and $64$ original steps, respectively. This experiment averages existing per-step captures; it does not recompute PPO minibatch updates or shuffle steps across windows.

\textbf{Training and evaluation.} We use an architecture and losses similar to those in \Cref{tab:hyperparameters}, with action histograms replacing single labels for cross-entropy supervision.
Four-step training uses four GPUs and $3$ accumulation steps, giving an effective batch size of $96$. The from-scratch attacker follows the $50$-epoch schedule. For fine-tuning, we initialize from an attacker trained on unaugmented per-step captures, changing the duration to $20$ epochs, learning rate to $10^{-5}$, learning-rate warmup to $2$ epochs, and minimum learning rate to $10^{-7}$. 
All three four-step evaluations use autoregressive inference on the same first $100$ held-out aggregate sequences with stride $8$, yielding $800$ last-frame targets and action histograms. For eight-step averaging, the from-scratch run uses the same model, losses, $50$-epoch schedule, and episode-disjoint splitting procedure. Evaluation uses autoregressive inference on $39$ held-out sequences with stride $8$, yielding $312$ last-frame targets and action histograms. Histogram error is measured by total variation distance, $\mathrm{TV}(\hat{q},q)=\frac{1}{2}\sum_{k=1}^{5}|\hat{q}_k-q_k|$, where $q_k$ is the observed fraction of action $k$ within the aggregation window and $\hat{q}_k$ is its predicted proportion. Lower TV indicates closer action proportions.

The \emph{uniform baseline} assigns $20\%$ probability to each of five actions in every window, without training or gradients. We divide action counts by the window length, compute the true histogram's TV distance from the uniform prediction, and average over the same evaluation windows as \sysname, separately for each aggregation size. This tests whether \sysname outperforms an equal-frequency guess, because improvement over the unadapted attacker alone does not establish an advantage over this no-gradient reference.

\input{tables/aggregation}

\textbf{Results.} For four-step averaging, \Cref{tab:aggregation} shows that fine-tuning reduces mean histogram TV from $0.576$ to $0.457$, while PSNR rises from $17.90$ to $18.04$\,dB and SSIM from $0.604$ to $0.611$. LPIPS does not improve, increasing slightly from $0.410$ to $0.412$. Fine-tuning still underperforms the uniform baseline, whose mean TV is $0.428$. Training from scratch reduces TV to $0.296$, approximately $31\%$ below uniform, with lower sequence-averaged TV in $94$ of $100$ held-out sequences. \Cref{fig:aggregation_action_error} shows these error distributions. Compared with fine-tuning, training from scratch recovers action proportions more accurately but produces less accurate image reconstructions, with PSNR of $17.03$\,dB and LPIPS of $0.510$. The runs use different training budgets and learning rates, so this comparison does not isolate the effect of initialization.

For eight-step averaging, training from scratch yields $16.03$\,dB PSNR and histogram TV of $0.268$, approximately $12\%$ below its uniform baseline of $0.304$. Training-set sizes, window targets, and evaluation sample counts differ between aggregation sizes, so these results do not isolate the effect of aggregation size. The image metrics quantify last-frame reconstruction under gradient averaging, not recovery of every observation or its order within a window. These single-run results do not establish robustness to shuffled minibatches, multi-epoch model updates, or secure aggregation.

Overall, training from scratch recovers action proportions more accurately than uniform guessing under both aggregation settings. For Agg4, however, it yields lower PSNR and SSIM and higher MSE and LPIPS than the pretrained attacker, both before and after fine-tuning. The models trained from scratch use $8{,}736$ aggregate training sequences for Agg4 and $2{,}808$ for Agg8. Learning image reconstruction from these smaller datasets without per-step pretraining may contribute to the lower visual fidelity. These experiments do not isolate the effects of dataset size, initialization, and training schedule.

\begin{figure}[t]
    \centering
    \begin{minipage}[c]{0.7\linewidth}
        \centering
        \begin{subfigure}[t]{0.61\linewidth}
            \centering
            \includegraphics[width=\linewidth]{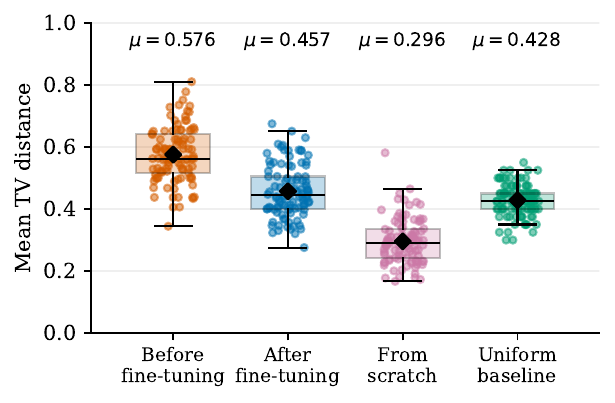}
            \caption{Agg4}
        \end{subfigure}\hfill
        \begin{subfigure}[t]{0.38\linewidth}
            \centering
            \includegraphics[width=\linewidth]{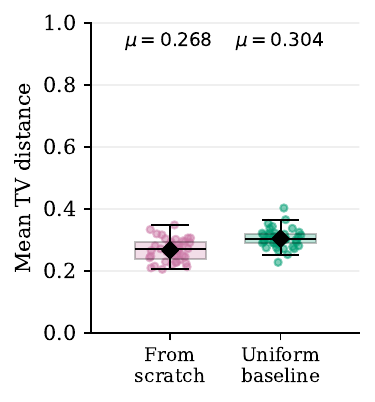}
            \caption{Agg8}
        \end{subfigure}
    \end{minipage}\hfill
    \begin{minipage}[c]{0.28\linewidth}
        \caption{Action-histogram TV for Agg4 and Agg8; lower is better. Reductions in mean TV relative to each panel's uniform baseline indicate improved action-proportion recovery. Each point averages eight windows in one sequence. Boxes show medians and interquartile ranges; diamonds mark means.}
        \label{fig:aggregation_action_error}
    \end{minipage}
    \vspace{-6pt}
\end{figure}

\subsection{Effect of Auxiliary Dataset Size}
\label{sec:auxiliary_scaling}

The few-shot study in \Cref{sec:few_shot_adaptation} varies the amount of target-scene data available for adaptation but does not quantify the auxiliary data required to train an attacker. To address this question, we vary the amount of unaugmented auxiliary training data used to train \sysname while holding the victim checkpoint and attacker architecture fixed.
The dataset comprises $49{,}800$ sequences, of which $2{,}487$ constitute a fixed validation set. We evaluate each attacker on $800$ images drawn from the same first $100$ sequences of the held-out per-step dataset.

\input{tables/auxiliary_scaling}

\textbf{Results.} \Cref{tab:auxiliary_scaling} reports the completed $10\%$, $30\%$, $50\%$, $70\%$, $90\%$, and $100\%$ runs. PSNR increases from $16.25$ to $17.66$, $18.24$, $18.44$, $18.60$, and $18.71$\,dB, while LPIPS falls from $0.502$ to $0.419$, $0.389$, $0.373$, $0.359$, and $0.354$. Action accuracy is $100\%$ in all six evaluations. Thus, action recovery remains accurate at the smallest tested fraction, while visual reconstruction benefits from more auxiliary data under this fixed-epoch training protocol. These single-run results do not identify the minimum auxiliary-data requirement or establish robustness to a different task or data domain.

\subsection{Qualitative Defense Example}
\label{sec:ext_defense_evaluation}

This section complements the defense analysis in \Cref{sec:defense_evaluation} with a qualitative example under representative defenses. \Cref{fig:defense} shows the reconstructed images for $T{=}8$ under different defense techniques. For this evaluation, we employ the $20\%$ adapted model of the PPO victim to present a visual assessment of the defense's effectiveness.

\begin{figure}[t]
    \centering
    \includegraphics[width=\linewidth]{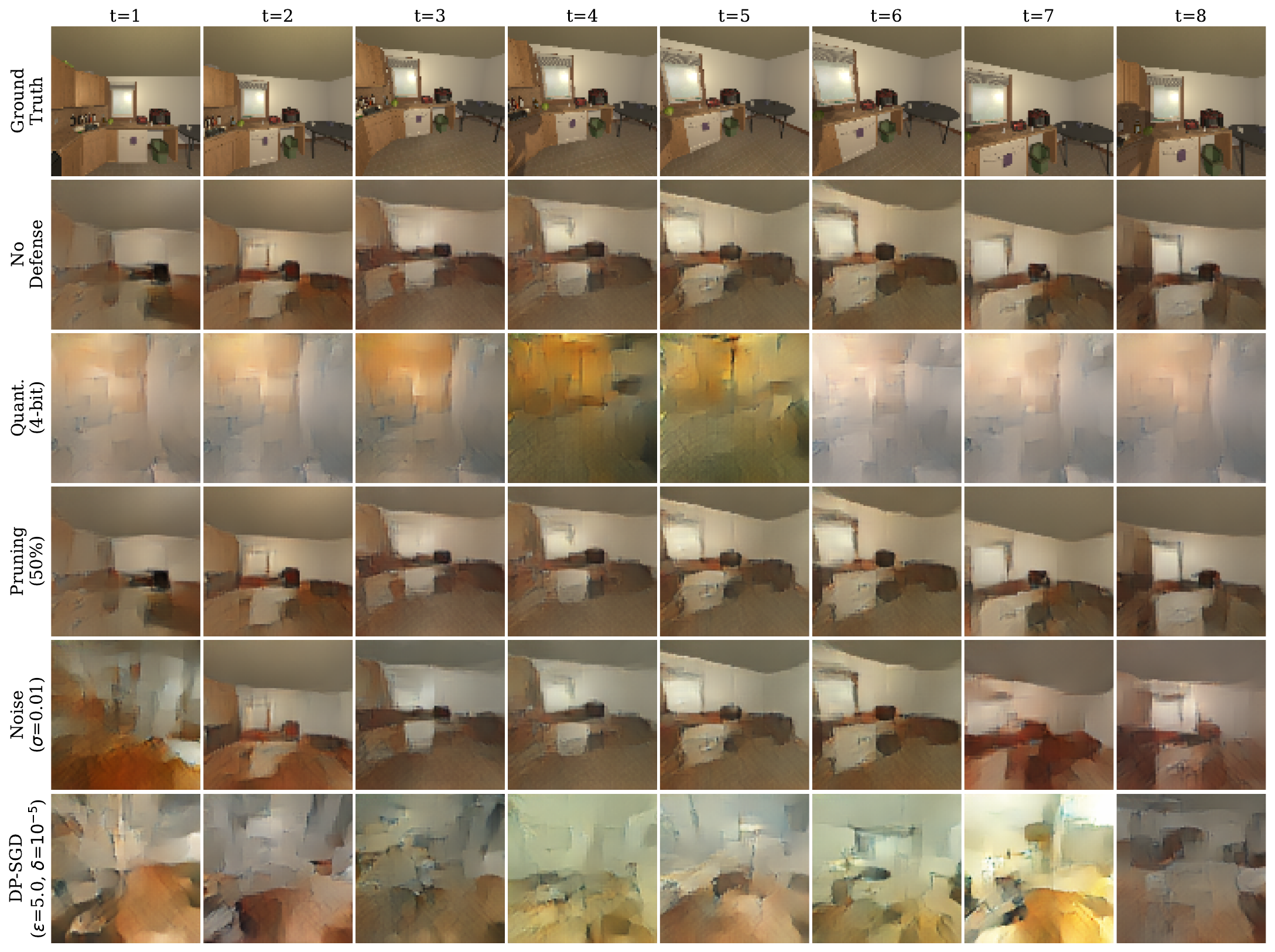}  
    \caption{
    Reconstructions by \sysname under different defense techniques (20\%--adapted PPO attacker; corresponding zero-shot results are reported in \Cref{tab:defenses}).
    }
    \label{fig:defense}
    \vspace{-6pt}
\end{figure}

\subsection{Extended Ablation Study}
\label{sec:ext_ablations}

\input{tables/a2c_ablation}

\Cref{tab:ablations_a2c} repeats the main-body ablations of \Cref{tab:ablations} with an A2C victim and makes three observations concrete. \emph{(i) Context window.} Enlarging $T$ from $1$ to $8$ raises PSNR from $15.54$ to $18.60$\,dB ($+3.1$\,dB), after which A2C plateaus in $[18.26, 18.77]$\,dB for $T\in\{16,32,64\}$, in contrast to PPO, which continues to improve to $20.29$\,dB at $T{=}64$ (\Cref{tab:ablations}). This division is consistent with algorithm-dependent temporal signal strength, while the exact sufficiency result in \Cref{lem:gradient_sufficiency} serves as an idealized reference for these practical GAE-based estimators. One possible explanation is that A2C gradients are more strongly tied to the current state, making additional context less informative after short horizons, whereas PPO's use of GAE and rollout-level normalization may preserve cross-timestep dependencies that longer contexts can exploit. \emph{(ii) Gradient components.} CNN-only and FC-only subsets recover PSNR within roughly $0.5$\,dB of the full gradient, while the policy/value heads alone drop $2.6$\,dB yet still achieve $100.0\%$ action accuracy---mirroring the PPO pattern in \Cref{tab:ablations} and aligning with \Cref{thm:action_identifiability}, which predicts exact action recovery from the policy-head structure in our $K{=}5$ setting. \emph{(iii) Loss ablations.} Removing the action loss collapses action accuracy to $16.4\%$ without affecting image PSNR, confirming that the two heads learn complementary signals. Removing LPIPS produces a $0.59$\,dB PSNR drop for A2C, whereas the PPO ablation (\Cref{tab:ablations}) shows a slight PSNR increase upon removing LPIPS---consistent with the standard perception--distortion tradeoff. We attribute this discrepancy to the different loss landscapes of the two algorithms: the A2C model appears to rely more heavily on the perceptual loss for stable optimization, though the precise mechanism warrants further investigation.

%% file: tables/hyperparameters.tex
\begin{table}[t]
    \centering
    \caption{Hyperparameter settings for \sysname training.}
    \label{tab:hyperparameters}
    \small
    \begin{tabular}{@{}ll@{}}
        \toprule
        \textbf{Hyperparameter} & \textbf{Value} \\
        \midrule
        Sequence length ($T$) & $8$ \\
        Latent dimension ($D$) & $1024$ \\
        Encoder hidden dimension ($d_{\mathrm{enc}}$) & $2048$ \\
        Transformer layers & $6$ \\
        Attention heads & $8$ \\
        Encoder / Decoder type & residual / residual \\
        \midrule
        Epochs & $50$ \\
        Batch size (per GPU) & $8$ \\
        Gradient accumulation steps & 2 \\
        Learning rate & $1 \times 10^{-4}$ \\
        Weight decay & $1 \times 10^{-5}$ \\
        LR schedule & cosine (warmup: 5 epochs, min LR: $1 \times 10^{-6}$) \\
        \midrule
        Scheduled sampling range & $[0.1, 0.8]$ (warmup: $5$ epochs) \\
        Rollout steps & $3$ \\
        \midrule
        $\alpha_1$ (MSE) & $1.0$ \\
        $\alpha_2$ (L1) & $0.5$ \\
        $\alpha_3$ (Action Cross Entropy) & $0.01$ \\
        $\alpha_4$ (LPIPS) & $0.2$ \\
        $\alpha_5$ (Rollout) & $0.5$ \\
        \bottomrule
    \end{tabular}
    \vspace{-6pt}
\end{table}

%% file: tables/runtime.tex
\begin{table}[t]
    \centering
    \setlength{\tabcolsep}{2.5pt}
    \caption{Runtime of \sysname across various training configurations. \textit{Seq.} indicates the number of gradient sequences used per run. Differences between PPO and A2C arise from the differing sizes of their captured gradient datasets. Reported times, measured in wall-clock hours, encompass training, validation every fifth epoch, and final evaluation on 100 held-out sequences. Unless otherwise specified, reported timings use six NVIDIA H200 GPUs; few-shot adaptation uses two NVIDIA H200 GPUs due to the smaller dataset. These timings reflect our measurement hardware rather than a strict hardware requirement.}
    \label{tab:runtime}
    \small
    \resizebox{\linewidth}{!}{%
        \begin{tabular}{@{}llccccccgg@{}}
            \toprule
            \multirow{2}{*}{\textbf{Experiment}}
            & \multirow{2}{*}{\textbf{Setting}}
            & \multirow{2}{*}{$d_{\mathrm{enc}}$}
            & \multirow{2}{*}{$D$}
            & \multirow{2}{*}{$T$}
            & \multirow{2}{*}{\textbf{Epochs}}
            & \multicolumn{2}{c}{\textbf{PPO}}
            & \multicolumn{2}{c}{\textbf{A2C}} \\
            \cmidrule(lr){7-8}
            \cmidrule(lr){9-10}
            &&&&&& \textbf{Seq.} & \textbf{Time (h)} & \textbf{Seq.} & \textbf{Time (h)} \\
            \midrule
            Base
                & Full Gradients & $2048$ & $1024$ & $8$ & $50$ & $127{,}188$ & $10.57$ & $149{,}250$ & $12.92$ \\
            \midrule
            \multirow{5}{*}{\makecell{Sequence\\length\\(\Cref{tab:ablations,tab:ablations_a2c})}}
                & $T{=}1$ (single-frame) & $1024$ & $1024$ & $1$ & $50$ & $137{,}013$ & $2.63$ & $162{,}301$ & $3.43$ \\
                & $T{=}8$ & $1024$ & $1024$ & $8$ & $50$ & $127{,}188$ & $8.92$ & $149{,}250$ & $12.26$ \\
                & $T{=}16$ & $1024$ & $1024$ & $16$ & $50$ & $115{,}161$ & $13.79$ & $134{,}589$ & $18.84$ \\
                & $T{=}32$ & $1024$ & $1024$ & $32$ & $50$ & $47{,}541$ & $16.37$ & $54{,}836$ & $19.15$ \\
                & $T{=}64$ (RoPE) & $1024$ & $1024$ & $64$ & $50$ & $12{,}966$ & $17.62$ & $14{,}600$ & $20.24$ \\
            \midrule
            \multirow{3}{*}{\makecell{Gradient\\component\\(\Cref{tab:ablations,tab:ablations_a2c})}}
                & CNN layers only & $1024$ & $1024$ & $8$ & $50$ & $127{,}188$ & $6.69$ & $149{,}250$ & $10.18$ \\
                & FC layer only & $1024$ & $1024$ & $8$ & $50$ & $127{,}188$ & $8.59$ & $149{,}250$ & $11.50$ \\
                & Policy \& value heads only & $1024$ & $1024$ & $8$ & $50$ & $127{,}188$ & $6.50$ & $149{,}250$ & $9.93$ \\
            \midrule
            \multirow{4}{*}{\makecell{Loss\\ablation\\(\Cref{tab:ablations,tab:ablations_a2c})}}
                & w/o MSE & $1024$ & $1024$ & $8$ & $50$ & $127{,}188$ & $8.92$ & $149{,}250$ & $11.97$ \\
                & w/o L1 & $1024$ & $1024$ & $8$ & $50$ & $127{,}188$ & $9.17$ & $149{,}250$ & $11.91$ \\
                & w/o LPIPS & $1024$ & $1024$ & $8$ & $50$ & $127{,}188$ & $8.37$ & $149{,}250$ & $11.94$ \\
                & w/o Action & $1024$ & $1024$ & $8$ & $50$ & $127{,}188$ & $8.91$ & $149{,}250$ & $11.85$ \\
            \midrule
            \multirow{5}{*}{\makecell{Few-shot\\adaptation\\(\Cref{tab:ppo_fewshot,tab:a2c_fewshot})}}
                & $10\%$ target data & $2048$ & $1024$ & $8$ & $20$ & $531$ & $0.35$ & $1{,}611$ & $0.54$ \\
                & $20\%$ target data & $2048$ & $1024$ & $8$ & $20$ & $1{,}181$ & $0.44$ & $3{,}510$ & $0.67$ \\
                & $30\%$ target data & $2048$ & $1024$ & $8$ & $20$ & $1{,}831$ & $0.51$ & $5{,}431$ & $0.85$ \\
                & $40\%$ target data & $2048$ & $1024$ & $8$ & $20$ & $2{,}359$ & $0.53$ & $7{,}046$ & $0.99$ \\
                & $50\%$ target data & $2048$ & $1024$ & $8$ & $20$ & $3{,}022$ & $0.61$ & $8{,}988$ & $1.18$ \\
            \bottomrule
        \end{tabular}%
    }
    \vspace{-6pt}
\end{table}

%% file: tables/a2c_adaptation.tex
\begin{wraptable}{r}{0.5\textwidth}
    \centering
    \setlength{\tabcolsep}{2pt}
    \vspace{-14pt}
    \caption{Few-shot adaptation on held-out scenes (A2C).}
    \vspace{-3pt}
    \label{tab:a2c_fewshot}
    \resizebox{\linewidth}{!}{%
        \begin{tabular}{@{}cccccc@{}}
            \hiderowcolors
            \toprule
            $p$ (\%) & \textbf{MSE}$\downarrow$ & \textbf{PSNR}$\uparrow$ & \textbf{SSIM}$\uparrow$ & \textbf{LPIPS}$\downarrow$ & \textbf{Act. Acc.}$\uparrow$ \\
            \midrule
            \showrowcolors
            0 & $0.014 {\scriptstyle \pm 0.005}$ & $18.9 {\scriptstyle \pm 1.4}$ & $0.583 {\scriptstyle \pm 0.061}$ & $0.358 {\scriptstyle \pm 0.060}$ & $100.0\%$ \\
            \hiderowcolors
            10 & $0.009 {\scriptstyle \pm 0.007}$ & $21.2 {\scriptstyle \pm 2.6}$ & $0.649 {\scriptstyle \pm 0.085}$ & $0.252 {\scriptstyle \pm 0.087}$ & $100.0\%$ \\
            \showrowcolors
            20 & $0.006 {\scriptstyle \pm 0.004}$ & $22.7 {\scriptstyle \pm 2.4}$ & $0.680 {\scriptstyle \pm 0.087}$ & $0.208 {\scriptstyle \pm 0.083}$ & $100.0\%$ \\
            \hiderowcolors
            30 & $0.005 {\scriptstyle \pm 0.003}$ & $23.6 {\scriptstyle \pm 2.6}$ & $0.703 {\scriptstyle \pm 0.092}$ & $0.173 {\scriptstyle \pm 0.080}$ & $100.0\%$ \\
            \showrowcolors
            40 & $0.004 {\scriptstyle \pm 0.002}$ & $24.4 {\scriptstyle \pm 2.9}$ & $0.722 {\scriptstyle \pm 0.098}$ & $0.148 {\scriptstyle \pm 0.080}$ & $100.0\%$ \\
            \hiderowcolors
            50 & $0.004 {\scriptstyle \pm 0.002}$ & $25.0 {\scriptstyle \pm 3.1}$ & $0.738 {\scriptstyle \pm 0.099}$ & $0.132 {\scriptstyle \pm 0.077}$ & $100.0\%$ \\
            \bottomrule
        \end{tabular}
    }
    \vspace{-10pt}
\end{wraptable}

%% file: tables/victim_variant_setup.tex
\begin{table}[t]
    \centering
    \small
    \setlength{\tabcolsep}{5pt}
    \caption{
    Runtime of \sysname for the additional PPO attackers. The full-gradient dimension $d$ equals the number of trainable parameters in the victim model. Reported wall-clock times, in hours, include training, validation every five epochs, and final evaluation on 100 held-out sequences. \textit{Seq.} denotes the number of training sequences, excluding the $5\%$ validation split. Attacker sizes are reported in billions of parameters. All GPUs are NVIDIA H200s.}
    \label{tab:victim_variant_setup}
    \begin{tabular}{@{}lrrrrrr@{}}
        \hiderowcolors
        \toprule
        Victim & $K$ & $d$ & Attacker (B) & Seq. & GPUs & Time (h) \\
        \midrule
        \showrowcolors
        CNN & 10 & 938,667 & 2.278 & 46,871 & 6 & 5.29 \\
        \hiderowcolors
        Tiny ViT & 5 & 477,830 & 1.335 & 46,770 & 4 & 5.46 \\
        \showrowcolors
        IMPALA-style CNN & 5 & 1,090,790 & 2.590 & 47,118 & 4 & 6.85 \\
        \hiderowcolors
        Wider IMPALA CNN & 5 & 2,373,318 & 2.180 & 47,098 & 6 & 12.19 \\
        \showrowcolors
        CNN+GRU & 5 & 919,462 & 2.239 & 47,215 & 6 & 7.71 \\
        \hiderowcolors
        Multi-modal CNN & 5 & 1,281,318 & 2.980 & 39,859 & 6 & 5.85 \\
        \bottomrule
    \end{tabular}
    \vspace{-6pt}
\end{table}

%% file: tables/victim_variants.tex
\begin{table}[t]
    \centering
    \small
    \setlength{\tabcolsep}{4pt}
    \caption{Additional PPO victims, each with a separately trained attacker and no data augmentation. $K$ is the number of actions.}
    \label{tab:victim_variants}
    \begin{tabular}{@{}lrrrrrr@{}}
        \hiderowcolors
        \toprule
        Victim & $K$ & MSE$\downarrow$ & PSNR (dB)$\uparrow$ & SSIM$\uparrow$ & LPIPS$\downarrow$ & Act. Acc. (\%)$\uparrow$ \\
        \midrule
        \showrowcolors
        CNN & 10 & $0.013{\scriptstyle\pm0.003}$ & $18.92{\scriptstyle\pm0.96}$ & $0.610{\scriptstyle\pm0.051}$ & $0.372{\scriptstyle\pm0.074}$ & 99.5 \\
        \hiderowcolors
        Tiny ViT & 5 & $0.048{\scriptstyle\pm0.028}$ & $13.99{\scriptstyle\pm2.81}$ & $0.456{\scriptstyle\pm0.096}$ & $0.543{\scriptstyle\pm0.057}$ & 99.0 \\
        \showrowcolors
        IMPALA-style CNN & 5 & $0.007{\scriptstyle\pm0.003}$ & $21.85{\scriptstyle\pm1.93}$ & $0.725{\scriptstyle\pm0.103}$ & $0.244{\scriptstyle\pm0.057}$ & 97.4 \\
        \hiderowcolors
        Wider IMPALA CNN & 5 & $0.007{\scriptstyle\pm0.005}$ & $22.36{\scriptstyle\pm3.05}$ & $0.773{\scriptstyle\pm0.127}$ & $0.208{\scriptstyle\pm0.106}$ & 98.5 \\
        \showrowcolors
        CNN+GRU & 5 & $0.015{\scriptstyle\pm0.006}$ & $18.62{\scriptstyle\pm1.50}$ & $0.565{\scriptstyle\pm0.089}$ & $0.390{\scriptstyle\pm0.072}$ & 97.5 \\
        \hiderowcolors
        Multi-modal CNN & 5 & $0.011{\scriptstyle\pm0.004}$ & $19.78{\scriptstyle\pm1.64}$ & $0.697{\scriptstyle\pm0.076}$ & $0.335{\scriptstyle\pm0.081}$ & 97.4 \\
        \bottomrule
    \end{tabular}
    \vspace{-6pt}
\end{table}

%% file: tables/aggregation.tex
\begin{table}[t]
    \centering
    \small
    \setlength{\tabcolsep}{4pt}
    \caption{Reconstruction from four- and eight-step gradient averages for the five-action PPO victim. Uniform baselines have no image output.}
    \label{tab:aggregation}
    \begin{tabular}{@{}lrrrrr@{}}
        \hiderowcolors
        \toprule
        Attacker & MSE$\downarrow$ & PSNR (dB)$\uparrow$ & SSIM$\uparrow$ & LPIPS$\downarrow$ & Hist. TV$\downarrow$ \\
        \midrule
        \multicolumn{6}{@{}l}{\textit{Agg4: Four-step averaging}} \\
        [3pt]
        Before fine-tuning & $0.022{\scriptstyle\pm0.015}$ & $17.90{\scriptstyle\pm3.80}$ & $0.604{\scriptstyle\pm0.148}$ & $0.410{\scriptstyle\pm0.097}$ & $0.576{\scriptstyle\pm0.226}$ \\
        After fine-tuning & $0.021{\scriptstyle\pm0.014}$ & $18.04{\scriptstyle\pm3.60}$ & $0.611{\scriptstyle\pm0.144}$ & $0.412{\scriptstyle\pm0.097}$ & $0.457{\scriptstyle\pm0.189}$ \\
        From scratch & $0.025{\scriptstyle\pm0.016}$ & $17.03{\scriptstyle\pm3.23}$ & $0.578{\scriptstyle\pm0.140}$ & $0.510{\scriptstyle\pm0.085}$ & $0.296{\scriptstyle\pm0.161}$ \\
        Uniform baseline & -- & -- & -- & -- & $0.428{\scriptstyle\pm0.140}$ \\
        \midrule
        \multicolumn{6}{@{}l}{\textit{Agg8: Eight-step averaging}} \\[3pt]
        From scratch & $0.031{\scriptstyle\pm0.019}$ & $16.03{\scriptstyle\pm3.31}$ & $0.568{\scriptstyle\pm0.142}$ & $0.559{\scriptstyle\pm0.096}$ & $0.268{\scriptstyle\pm0.115}$ \\
        Uniform baseline & -- & -- & -- & -- & $0.304{\scriptstyle\pm0.094}$ \\
        \bottomrule
    \end{tabular}
    \vspace{-6pt}
\end{table}

%% file: tables/auxiliary_scaling.tex
\begin{table}[t]
    \centering
    \small
    \setlength{\tabcolsep}{4pt}
    \caption{Auxiliary-data scaling without augmentation for PPO. Percentages select training episodes after a fixed validation holdout; \textit{Seq.} counts training sequences.}
    \label{tab:auxiliary_scaling}
    \begin{tabular}{@{}lrrrrrr@{}}
        \hiderowcolors
        \toprule
        Auxiliary data & Seq. & MSE$\downarrow$ & PSNR (dB)$\uparrow$ & SSIM$\uparrow$ & LPIPS$\downarrow$ & Act. Acc. (\%)$\uparrow$ \\
        \midrule
        \showrowcolors
        $10\%$ & 4,846 & $0.026{\scriptstyle\pm0.010}$ & $16.25{\scriptstyle\pm1.80}$ & $0.548{\scriptstyle\pm0.104}$ & $0.502{\scriptstyle\pm0.063}$ & 100.0 \\
        \hiderowcolors
        $30\%$ & 14,267 & $0.019{\scriptstyle\pm0.007}$ & $17.66{\scriptstyle\pm1.78}$ & $0.594{\scriptstyle\pm0.104}$ & $0.419{\scriptstyle\pm0.062}$ & 100.0 \\
        \showrowcolors
        $50\%$ & 23,694 & $0.016{\scriptstyle\pm0.006}$ & $18.24{\scriptstyle\pm1.65}$ & $0.609{\scriptstyle\pm0.108}$ & $0.389{\scriptstyle\pm0.072}$ & 100.0 \\
        \hiderowcolors
        $70\%$ & 32,981 & $0.016{\scriptstyle\pm0.007}$ & $18.44{\scriptstyle\pm1.78}$ & $0.617{\scriptstyle\pm0.105}$ & $0.373{\scriptstyle\pm0.074}$ & 100.0 \\
        \showrowcolors
        $90\%$ & 42,329 & $0.015{\scriptstyle\pm0.005}$ & $18.60{\scriptstyle\pm1.67}$ & $0.625{\scriptstyle\pm0.100}$ & $0.359{\scriptstyle\pm0.065}$ & 100.0 \\
        \hiderowcolors
        $100\%$ & 47,313 & $0.014{\scriptstyle\pm0.005}$ & $18.71{\scriptstyle\pm1.62}$ & $0.627{\scriptstyle\pm0.102}$ & $0.354{\scriptstyle\pm0.074}$ & 100.0 \\
        \bottomrule
    \end{tabular}
    \vspace{-6pt}
\end{table}

%% file: tables/a2c_ablation.tex
\begin{table}[t]
    \centering
    \setlength{\tabcolsep}{2.5pt}
    \caption{Ablation study conducted on the held-out set (A2C).}
    \label{tab:ablations_a2c}
    \small
        \begin{tabular}{@{}llccccc@{}}
            \toprule
             & \textbf{Settings} & \textbf{MSE}$\downarrow$ & \textbf{PSNR}$\uparrow$ & \textbf{SSIM}$\uparrow$ & \textbf{LPIPS}$\downarrow$ & \textbf{Act. Acc.}$\uparrow$ \\
            \midrule
            \multirow{6}{*}{\rotatebox{90}{$T$}}
                & 1 & $0.028 {\scriptstyle \pm 0.004}$ & $15.54 {\scriptstyle \pm 0.69}$ & $0.483 {\scriptstyle \pm 0.034}$ & $0.695 {\scriptstyle \pm 0.064}$ & 99.0\% \\
                & 8 & $0.014 {\scriptstyle \pm 0.004}$ & $18.60 {\scriptstyle \pm 1.23}$ & $0.564 {\scriptstyle \pm 0.060}$ & $0.388 {\scriptstyle \pm 0.064}$ & 100.0\% \\
                & 16 & $0.015 {\scriptstyle \pm 0.005}$ & $18.60 {\scriptstyle \pm 1.31}$ & $0.579 {\scriptstyle \pm 0.061}$ & $0.375 {\scriptstyle \pm 0.055}$ & 98.4\% \\
                & 32 & $0.014 {\scriptstyle \pm 0.004}$ & $18.77 {\scriptstyle \pm 1.29}$ & $0.541 {\scriptstyle \pm 0.068}$ & $0.390 {\scriptstyle \pm 0.053}$ & 98.9\% \\
                & 64 & $0.016 {\scriptstyle \pm 0.007}$ & $18.26 {\scriptstyle \pm 1.75}$ & $0.560 {\scriptstyle \pm 0.084}$ & $0.443 {\scriptstyle \pm 0.061}$ & 98.9\% \\
            \midrule
            \multirow{3}{*}{\rotatebox{90}{Comp.}}
                & \cellcolor{gray!10}{CNN} & \cellcolor{gray!10}{$0.016 {\scriptstyle \pm 0.005}$} & \cellcolor{gray!10}{$18.15 {\scriptstyle \pm 1.20}$} & \cellcolor{gray!10}{$0.553 {\scriptstyle \pm 0.061}$} & \cellcolor{gray!10}{$0.391 {\scriptstyle \pm 0.067}$} & \cellcolor{gray!10}{99.5\%} \\
                & \cellcolor{gray!10}{FC} & \cellcolor{gray!10}{$0.016 {\scriptstyle \pm 0.007}$} & \cellcolor{gray!10}{$18.23 {\scriptstyle \pm 1.43}$} & \cellcolor{gray!10}{$0.558 {\scriptstyle \pm 0.063}$} & \cellcolor{gray!10}{$0.405 {\scriptstyle \pm 0.069}$} & \cellcolor{gray!10}{99.5\%} \\
                & \cellcolor{gray!10}{Heads} & \cellcolor{gray!10}{$0.028 {\scriptstyle \pm 0.013}$} & \cellcolor{gray!10}{$15.97 {\scriptstyle \pm 1.92}$} & \cellcolor{gray!10}{$0.500 {\scriptstyle \pm 0.074}$} & \cellcolor{gray!10}{$0.431 {\scriptstyle \pm 0.091}$} & \cellcolor{gray!10}{100.0\%} \\
            \midrule
            \multirow{4}{*}{\rotatebox{90}{Loss (w/o)}}
                & Act. & $0.013 {\scriptstyle \pm 0.004}$ & $18.88 {\scriptstyle \pm 1.17}$ & $0.573 {\scriptstyle \pm 0.058}$ & $0.386 {\scriptstyle \pm 0.054}$ & 16.4\% \\
                & L1 & $0.014 {\scriptstyle \pm 0.004}$ & $18.68 {\scriptstyle \pm 1.23}$ & $0.568 {\scriptstyle \pm 0.057}$ & $0.388 {\scriptstyle \pm 0.066}$ & 100.0\% \\
                & LPIPS & $0.017 {\scriptstyle \pm 0.007}$ & $18.01 {\scriptstyle \pm 1.47}$ & $0.567 {\scriptstyle \pm 0.060}$ & $0.409 {\scriptstyle \pm 0.067}$ & 100.0\% \\
                & MSE & $0.015 {\scriptstyle \pm 0.004}$ & $18.51 {\scriptstyle \pm 1.10}$ & $0.564 {\scriptstyle \pm 0.062}$ & $0.382 {\scriptstyle \pm 0.064}$ & 100.0\% \\
            \bottomrule
        \end{tabular}%
    \vspace{-6pt}
\end{table}

%% file: appendices/limitations.tex
\section{Discussion}
\label{sec:discussion}

\sysname is designed to reconstruct private trajectories from ordered per-step gradients in near real-time scenarios. Consequently, our evaluation focuses on practical defense mechanisms compatible with these conditions, including pruning, quantization, additive noise, and DP-SGD~\citep{abadi2016deep}. DP-SGD reduces reconstruction quality most strongly in our experiments, but its computational costs and privacy--utility trade-off remain unevaluated. We do not consider heavier defenses such as CENSOR~\citep{zhang2025censor} and SVDefense~\citep{luo2025svdefense}, whose overhead are challenging for real-time deployment.
\Cref{sec:aggregation} evaluates averages of four or eight consecutive stored gradients with last-frame and action-histogram targets, rather than recovery of every observation or its within-window order. These experiments retain ordering across aggregate windows; they do not evaluate shuffled minibatches, recomputed rollout-level updates, multi-epoch client policy updates, or multi-client secure aggregation. These mechanisms combine gradients across time steps, optimization states, or clients, potentially diminishing or eliminating the ordered temporal signal exploited by \sysname.

\section{Limitations}
\label{sec:limitations}

Our experiments remain confined to point-goal navigation in AI2-THOR~\citep{kolve2017ai2}. The additional PPO victims cover recurrent, residual, and compact transformer policies, multi-modal inputs, and up to ten discrete actions, but require a separately trained attacker for each victim. They do not establish cross-task or cross-simulator transfer, nor scalability to substantially larger policies. The multi-modal experiment reconstructs RGB and actions, not the goal input; depth, proprioception, and off-policy algorithms remain unevaluated. Recurrent captures condition on detached hidden states rather than backpropagating through complete histories. These settings retain the ordered per-step gradient interface and do not resolve the aggregation limitation above. Finally, our analysis targets discrete-action victims: \Cref{thm:action_identifiability} applies to a categorical policy over a finite set of actions.
For Gaussian continuous-control policies, gradients may reveal advantage-weighted residual structure, but identifiability depends on both the covariance parameterization and the advantage-sign information. Consequently, the categorical argmin/argmax rule does not apply directly, and the corresponding identifiability statement is left to future work.
\sysname trades offline training on auxiliary trajectories from the same task family for rapid online reconstruction. Our defense study evaluates reconstruction leakage, while downstream navigation utility remains to be assessed.

%% file: references.bib
@inproceedings{hatamizadeh2022gradvit,
  title={Gradvit: Gradient inversion of vision transformers},
  author={Hatamizadeh, Ali and Yin, Hongxu and Roth, Holger R and Li, Wenqi and Kautz, Jan and Xu, Daguang and Molchanov, Pavlo},
  booktitle={Proceedings of the IEEE/CVF conference on computer vision and pattern recognition},
  pages={10021--10030},
  year={2022}
}

@article{shi2023scale,
  title={Scale-mia: A scalable model inversion attack against secure federated learning via latent space reconstruction},
  author={Shi, Shanghao and Wang, Ning and Xiao, Yang and Zhang, Chaoyu and Shi, Yi and Hou, Y Thomas and Lou, Wenjing},
  journal={arXiv preprint arXiv:2311.05808},
  year={2023}
}

@inproceedings{lu2022april,
  title={April: Finding the achilles' heel on privacy for vision transformers},
  author={Lu, Jiahao and Zhang, Xi Sheryl and Zhao, Tianli and He, Xiangyu and Cheng, Jian},
  booktitle={Proceedings of the IEEE/CVF conference on computer vision and pattern recognition},
  pages={10051--10060},
  year={2022}
}

@inproceedings{yin2021see,
  title={See through gradients: Image batch recovery via gradinversion},
  author={Yin, Hongxu and Mallya, Arun and Vahdat, Arash and Alvarez, Jose M and Kautz, Jan and Molchanov, Pavlo},
  booktitle={Proceedings of the IEEE/CVF conference on computer vision and pattern recognition},
  pages={16337--16346},
  year={2021}
}

@article{schulman2017proximal,
  title={Proximal policy optimization algorithms},
  author={Schulman, John and Wolski, Filip and Dhariwal, Prafulla and Radford, Alec and Klimov, Oleg},
  journal={arXiv preprint arXiv:1707.06347},
  year={2017}
}

@inproceedings{carletti2025sok,
  title={$\{$SoK$\}$: Gradient Inversion Attacks in Federated Learning},
  author={Carletti, Vincenzo and Foggia, Pasquale and Mazzocca, Carlo and Parrella, Giuseppe and Vento, Mario},
  booktitle={34th USENIX Security Symposium (USENIX Security 25)},
  pages={6439--6459},
  year={2025}
}

@article{zhang2025censor,
  title={Censor: Defense against gradient inversion via orthogonal subspace Bayesian sampling},
  author={Zhang, Kaiyuan and Cheng, Siyuan and Shen, Guangyu and Ribeiro, Bruno and An, Shengwei and Chen, Pin-Yu and Zhang, Xiangyu and Li, Ninghui},
  journal={arXiv preprint arXiv:2501.15718},
  year={2025}
}

@inproceedings{li2022auditing,
  title={Auditing privacy defenses in federated learning via generative gradient leakage},
  author={Li, Zhuohang and Zhang, Jiaxin and Liu, Luyang and Liu, Jian},
  booktitle={Proceedings of the IEEE/CVF conference on computer vision and pattern recognition},
  pages={10132--10142},
  year={2022}
}

@article{geyer2017differentially,
  title={Differentially private federated learning: A client level perspective},
  author={Geyer, Robin C and Klein, Tassilo and Nabi, Moin},
  journal={arXiv preprint arXiv:1712.07557},
  year={2017}
}

@inproceedings{aji2017sparse,
  title={Sparse communication for distributed gradient descent},
  author={Aji, Alham Fikri and Heafield, Kenneth},
  booktitle={Proceedings of the 2017 conference on empirical methods in natural language processing},
  pages={440--445},
  year={2017}
}

@inproceedings{wei2021gradient,
  title={Gradient-leakage resilient federated learning},
  author={Wei, Wenqi and Liu, Ling and Wu, Yanzhao and Su, Gong and Iyengar, Arun},
  booktitle={2021 IEEE 41st International Conference on Distributed Computing Systems (ICDCS)},
  pages={797--807},
  year={2021},
  organization={IEEE}
}

@inproceedings{sun2021soteria,
  title={Soteria: Provable defense against privacy leakage in federated learning from representation perspective},
  author={Sun, Jingwei and Li, Ang and Wang, Binghui and Yang, Huanrui and Li, Hai and Chen, Yiran},
  booktitle={Proceedings of the IEEE/CVF conference on computer vision and pattern recognition},
  pages={9311--9319},
  year={2021}
}

@article{kolve2017ai2,
  title={Ai2-thor: An interactive 3d environment for visual ai},
  author={Kolve, Eric and Mottaghi, Roozbeh and Han, Winson and VanderBilt, Eli and Weihs, Luca and Herrasti, Alvaro and Deitke, Matt and Ehsani, Kiana and Gordon, Daniel and Zhu, Yuke and others},
  journal={arXiv preprint arXiv:1712.05474},
  year={2017}
}

@article{luo2025svdefense,
  title={SVDefense: Effective Defense against Gradient Inversion Attacks via Singular Value Decomposition},
  author={Luo, Chenxiang and Yau, David KY and Song, Qun},
  journal={arXiv preprint arXiv:2510.03319},
  year={2025}
}

@article{konda1999actor,
  title={Actor-critic algorithms},
  author={Konda, Vijay and Tsitsiklis, John},
  journal={Advances in neural information processing systems},
  volume={12},
  year={1999}
}

@article{zhu2019deep,
  title={Deep leakage from gradients},
  author={Zhu, Ligeng and Liu, Zhijian and Han, Song},
  journal={Advances in neural information processing systems},
  volume={32},
  year={2019}
}

@article{geiping2020inverting,
  title={Inverting gradients-how easy is it to break privacy in federated learning?},
  author={Geiping, Jonas and Bauermeister, Hartmut and Dr{\"o}ge, Hannah and Moeller, Michael},
  journal={Advances in neural information processing systems},
  volume={33},
  pages={16937--16947},
  year={2020}
}

@article{zhao2020idlg,
  title={idlg: Improved deep leakage from gradients},
  author={Zhao, Bo and Mopuri, Konda Reddy and Bilen, Hakan},
  journal={arXiv preprint arXiv:2001.02610},
  year={2020}
}

@article{jeon2021gradient,
  title={Gradient inversion with generative image prior},
  author={Jeon, Jinwoo and Lee, Kangwook and Oh, Sewoong and Ok, Jungseul and others},
  journal={Advances in neural information processing systems},
  volume={34},
  pages={29898--29908},
  year={2021}
}

@inproceedings{bonawitz2017practical,
  title={Practical secure aggregation for privacy-preserving machine learning},
  author={Bonawitz, Keith and Ivanov, Vladimir and Kreuter, Ben and Marcedone, Antonio and McMahan, H Brendan and Patel, Sarvar and Ramage, Daniel and Segal, Aaron and Seth, Karn},
  booktitle={proceedings of the 2017 ACM SIGSAC Conference on Computer and Communications Security},
  pages={1175--1191},
  year={2017}
}

@article{wang2004image,
  title={Image quality assessment: from error visibility to structural similarity},
  author={Wang, Zhou and Bovik, Alan C and Sheikh, Hamid R and Simoncelli, Eero P},
  journal={IEEE transactions on image processing},
  volume={13},
  number={4},
  pages={600--612},
  year={2004},
  publisher={IEEE}
}

@inproceedings{zhang2018unreasonable,
  title={The unreasonable effectiveness of deep features as a perceptual metric},
  author={Zhang, Richard and Isola, Phillip and Efros, Alexei A and Shechtman, Eli and Wang, Oliver},
  booktitle={Proceedings of the IEEE conference on computer vision and pattern recognition},
  pages={586--595},
  year={2018}
}

@inproceedings{li2024privacy,
  title={Privacy risks in reinforcement learning for household robots},
  author={Li, Miao and Ding, Wenhao and Zhao, Ding},
  booktitle={2024 IEEE International Conference on Robotics and Automation (ICRA)},
  pages={5148--5154},
  year={2024},
  organization={IEEE}
}

@article{wen2022fishing,
  title={Fishing for user data in large-batch federated learning via gradient magnification},
  author={Wen, Yuxin and Geiping, Jonas and Fowl, Liam and Goldblum, Micah and Goldstein, Tom},
  journal={arXiv preprint arXiv:2202.00580},
  year={2022}
}

@inproceedings{pasquini2022eluding,
  title={Eluding secure aggregation in federated learning via model inconsistency},
  author={Pasquini, Dario and Francati, Danilo and Ateniese, Giuseppe},
  booktitle={Proceedings of the 2022 ACM SIGSAC Conference on computer and communications security},
  pages={2429--2443},
  year={2022}
}

@article{fowl2021robbing,
  title={Robbing the fed: Directly obtaining private data in federated learning with modified models},
  author={Fowl, Liam and Geiping, Jonas and Czaja, Wojtek and Goldblum, Micah and Goldstein, Tom},
  journal={arXiv preprint arXiv:2110.13057},
  year={2021}
}

@inproceedings{zhao2024loki,
  title={Loki: Large-scale data reconstruction attack against federated learning through model manipulation},
  author={Zhao, Joshua C and Sharma, Atul and Elkordy, Ahmed Roushdy and Ezzeldin, Yahya H and Avestimehr, Salman and Bagchi, Saurabh},
  booktitle={2024 IEEE Symposium on Security and Privacy (SP)},
  pages={1287--1305},
  year={2024},
  organization={IEEE}
}

@inproceedings{wu2023learning,
  title={Learning to invert: Simple adaptive attacks for gradient inversion in federated learning},
  author={Wu, Ruihan and Chen, Xiangyu and Guo, Chuan and Weinberger, Kilian Q},
  booktitle={Uncertainty in Artificial Intelligence},
  pages={2293--2303},
  year={2023},
  organization={PMLR}
}

@article{du2023orl,
  title={Orl-auditor: Dataset auditing in offline deep reinforcement learning},
  author={Du, Linkang and Chen, Min and Sun, Mingyang and Ji, Shouling and Cheng, Peng and Chen, Jiming and Zhang, Zhikun},
  journal={arXiv preprint arXiv:2309.03081},
  year={2023}
}

@article{he2025gradient,
  title={Gradient Inversion in Federated Reinforcement Learning},
  author={He, Shenghong},
  journal={arXiv preprint arXiv:2512.00303},
  year={2025}
}

@article{huang2021evaluating,
  title={Evaluating gradient inversion attacks and defenses in federated learning},
  author={Huang, Yangsibo and Gupta, Samyak and Song, Zhao and Li, Kai and Arora, Sanjeev},
  journal={Advances in neural information processing systems},
  volume={34},
  pages={7232--7241},
  year={2021}
}

@inproceedings{abadi2016deep,
  title={Deep learning with differential privacy},
  author={Abadi, Martin and Chu, Andy and Goodfellow, Ian and McMahan, H Brendan and Mironov, Ilya and Talwar, Kunal and Zhang, Li},
  booktitle={Proceedings of the 2016 ACM SIGSAC conference on computer and communications security},
  pages={308--318},
  year={2016}
}

@article{bengio2015scheduled,
  title={Scheduled sampling for sequence prediction with recurrent neural networks},
  author={Bengio, Samy and Vinyals, Oriol and Jaitly, Navdeep and Shazeer, Noam},
  journal={Advances in neural information processing systems},
  volume={28},
  year={2015}
}

@article{ranzato2015sequence,
  title={Sequence level training with recurrent neural networks},
  author={Ranzato, Marc'Aurelio and Chopra, Sumit and Auli, Michael and Zaremba, Wojciech},
  journal={arXiv preprint arXiv:1511.06732},
  year={2015}
}

@article{schulman2015high,
  title={High-dimensional continuous control using generalized advantage estimation},
  author={Schulman, John and Moritz, Philipp and Levine, Sergey and Jordan, Michael and Abbeel, Pieter},
  journal={arXiv preprint arXiv:1506.02438},
  year={2015}
}

@article{dao2022flashattention,
  title={Flashattention: Fast and memory-efficient exact attention with io-awareness},
  author={Dao, Tri and Fu, Dan and Ermon, Stefano and Rudra, Atri and R{\'e}, Christopher},
  journal={Advances in neural information processing systems},
  volume={35},
  pages={16344--16359},
  year={2022}
}

@article{su2024roformer,
  title={Roformer: Enhanced transformer with rotary position embedding},
  author={Su, Jianlin and Ahmed, Murtadha and Lu, Yu and Pan, Shengfeng and Bo, Wen and Liu, Yunfeng},
  journal={Neurocomputing},
  volume={568},
  pages={127063},
  year={2024},
  publisher={Elsevier}
}

@inproceedings{mnih2016asynchronous,
  title={Asynchronous methods for deep reinforcement learning},
  author={Mnih, Volodymyr and Badia, Adria Puigdomenech and Mirza, Mehdi and Graves, Alex and Lillicrap, Timothy and Harley, Tim and Silver, David and Kavukcuoglu, Koray},
  booktitle={International conference on machine learning},
  pages={1928--1937},
  year={2016},
  organization={PmLR}
}

@article{williams1992simple,
  title={Simple statistical gradient-following algorithms for connectionist reinforcement learning},
  author={Williams, Ronald J},
  journal={Machine learning},
  volume={8},
  number={3},
  pages={229--256},
  year={1992},
  publisher={Springer}
}

@article{paszke2019pytorch,
  title={Pytorch: An imperative style, high-performance deep learning library},
  author={Paszke, Adam and Gross, Sam and Massa, Francisco and Lerer, Adam and Bradbury, James and Chanan, Gregory and Killeen, Trevor and Lin, Zeming and Gimelshein, Natalia and Antiga, Luca and others},
  journal={Advances in neural information processing systems},
  volume={32},
  year={2019}
}

@article{wang2009mean,
  title={Mean squared error: Love it or leave it? A new look at signal fidelity measures},
  author={Wang, Zhou and Bovik, Alan C},
  journal={IEEE signal processing magazine},
  volume={26},
  number={1},
  pages={98--117},
  year={2009},
  publisher={IEEE}
}

@article{lin2005visual,
  title={Visual distortion gauge based on discrimination of noticeable contrast changes},
  author={Lin, Weisi and Dong, Li and Xue, Ping},
  journal={IEEE transactions on circuits and systems for video technology},
  volume={15},
  number={7},
  pages={900--909},
  year={2005},
  publisher={IEEE}
}

@article{alistarh2017qsgd,
  title={QSGD: Communication-efficient SGD via gradient quantization and encoding},
  author={Alistarh, Dan and Grubic, Demjan and Li, Jerry and Tomioka, Ryota and Vojnovic, Milan},
  journal={Advances in neural information processing systems},
  volume={30},
  year={2017}
}

@inproceedings{zhou2025taming,
  title={Taming teacher forcing for masked autoregressive video generation},
  author={Zhou, Deyu and Sun, Quan and Peng, Yuang and Yan, Kun and Dong, Runpei and Wang, Duomin and Ge, Zheng and Duan, Nan and Zhang, Xiangyu},
  booktitle={Proceedings of the IEEE/CVF Conference on Computer Vision and Pattern Recognition},
  pages={7374--7384},
  year={2025}
}

@inproceedings{dosovitskiy2020image,
  title = {An {{Image}} Is {{Worth}} 16x16 {{Words}}: {{Transformers}} for {{Image Recognition}} at {{Scale}}},
  booktitle = {International {{Conference}} on {{Learning Representations}}},
  author = {Dosovitskiy, Alexey and Beyer, Lucas and Kolesnikov, Alexander and Weissenborn, Dirk and Zhai, Xiaohua and Unterthiner, Thomas and Dehghani, Mostafa and Minderer, Matthias and Heigold, Georg and Gelly, Sylvain and Uszkoreit, Jakob and Houlsby, Neil},
  year = 2020
}

@misc{espeholt2018impala,
  title = {{{IMPALA}}: {{Scalable Distributed Deep-RL}} with {{Importance Weighted Actor-Learner Architectures}}},
  author = {Espeholt, Lasse and Soyer, Hubert and Munos, Remi and Simonyan, Karen and Mnih, Volodymir and Ward, Tom and Doron, Yotam and Firoiu, Vlad and Harley, Tim and Dunning, Iain and Legg, Shane and Kavukcuoglu, Koray},
  year = 2018,
  eprint = {1802.01561},
  primaryclass = {cs.LG},
  doi = {10.48550/arXiv.1802.01561},
  archiveprefix = {arXiv}
}


%% file: software.bib
@software{hdf5,
author = {{The HDF Group}},
title = {{Hierarchical Data Format, version 5}},
url = {https://github.com/HDFGroup/hdf5}
}
